\documentclass{article}
\usepackage[utf8]{inputenc}

\usepackage{authblk}
\usepackage{hyperref}
    
\title{Stein Variational Gradient Descent:\\
many-particle and long-time asymptotics
}
\author[1]{Nikolas N\"usken}
\author[2]{D.R. Michiel Renger}
\date{\today}

\affil[1]{Institute of Mathematics, Universit\"at Potsdam, 14476 Potsdam, Germany, \href{mailto:nuesken@uni-potsdam.de}{nuesken@uni-potsdam.de}}
\affil[2]{WIAS Berlin, Mohrenstrasse 39, 10117 Berlin, Germany,
\href{mailto:renger@wias-berlin.de}{renger@wias-berlin.de}}
\usepackage{mathrsfs, amssymb}
\usepackage{amsthm}
\usepackage{amsmath}
\usepackage[intlimits]{empheq}  
\usepackage{bbm} 
\usepackage{xcolor}
\usepackage[scientific-notation=true]{siunitx}
\usepackage{float}
\usepackage{dsfont}	
\usepackage[numbers]{natbib}
\usepackage{subcaption}
\usepackage{stackrel}
\usepackage[inner=2cm,outer=2cm,top=2cm,bottom=2cm]{geometry}
\usepackage[ruled,vlined]{algorithm2e}
\usepackage{enumitem}
\usepackage[colorinlistoftodos,prependcaption,textsize=tiny]{todonotes}
\usepackage{nccmath}   
\usepackage{tensor}    

\DeclareMathOperator\Grad{grad}

\usepackage{color}

\newcommand{\wsconv}[1]{\stackrel[{#1}]{\ast\,\pi}{\rightharpoonup}}

\newcommand{\super}[1]{^{\scriptscriptstyle{(#1)}}}

\DeclareMathOperator*{\argmin}{arg\,min}

\newcommand{\dd}{\mathrm{d}}
\newcommand{\I}{\mathcal{I}}

\newcommand{\KK}{\mathbb{K}}
\newcommand{\Hcal}{\mathcal{H}}
\newcommand*{\Nb}{\mathbb{N}}

\newcommand{\T}{\mathcal{T}}
\newcommand{\Qcal}{\mathcal{Q}}
\newcommand{\PP}{\mathbb{P}}
\newcommand{\Pcal}{\mathcal{P}}
\newcommand{\Stein}{\mathrm{Stein}}

\newcommand{\A}{\mathcal{A}}
\newcommand{\Bcal}{\mathcal{B}}

\DeclareMathOperator*{\F}{\mathcal{F}}
\DeclareMathOperator*{\R}{\mathbb{R}}
\DeclareMathOperator*{\KL}{\operatorname{KL}}

\DeclareMathOperator{\Ham}{\mathcal{H}}

\theoremstyle{plain}
\newtheorem{theorem}{Theorem}[section]
\newtheorem{lemma}[theorem]{Lemma}
\newtheorem{corollary}[theorem]{Corollary}
\newtheorem{proposition}[theorem]{Proposition}
\newtheorem{assumption}{Assumption}

\newtheorem{result}[theorem]{Informal Result}

\theoremstyle{definition}
\newtheorem{definition}[theorem]{Definition}
\newtheorem{example}[theorem]{Example}

\theoremstyle{remark}
\newtheorem{remark}[theorem]{Remark}

\usepackage[T1]{fontenc}
\usepackage{lipsum}
\usepackage{tocloft}
\begin{document}
    
\maketitle
    
\begin{abstract}

Stein variational gradient descent (SVGD) refers to a class of methods for Bayesian inference based on interacting particle systems. In this paper, we consider the originally proposed  deterministic dynamics as well as a stochastic variant, each of which represent one of the two main paradigms in Bayesian computational statistics: \emph{variational inference} and \emph{Markov chain Monte Carlo}. As it turns out, these are tightly linked through a correspondence between gradient flow structures and large-deviation principles rooted in statistical physics. To expose this relationship, we develop the cotangent space construction for the Stein geometry, prove its basic properties, and determine the large-deviation functional governing the many-particle limit for the empirical measure. Moreover, we identify the \emph{Stein-Fisher information} (or \emph{kernelised Stein discrepancy}) as its leading order contribution in the long-time and many-particle regime in the sense of $\Gamma$-convergence, shedding some light on the  finite-particle properties of SVGD. Finally, we establish a comparison principle between the Stein-Fisher information and RKHS-norms that might be of independent interest.
\\

\end{abstract}

\section{Introduction}
\label{sec:intro}

Approximating high-dimensional probability distributions is a key challenge in many applications such as Bayesian inference or computational statistical physics. The \emph{target measure} of interest is typically given in the form
\begin{equation}
\label{eq:target}
  \pi = \frac{1}{Z} e^{-V} \, \mathrm{d}x
\end{equation}
in a high dimensional state space $\mathbb{R}^d$, where $Z = \int_{\mathbb{R}^d} e^{-V} \, \mathrm{d}x$ is a numerically intractable normalisation constant, and $V \in C^1(\mathbb{R}^d;\mathbb{R})$ is referred to as the \emph{potential}. Common algorithmic approaches can broadly be classified according to the following two paradigms:
\\\\
\emph{Variational inference} (VI) \cite{bishop2006pattern,blei2017variational,zhang2018advances} relies on a (parameterised) family of distributions $\mathcal{D} = \{\rho_{\phi}: \, \phi \in \Phi\}$, attempting to find an approximation $\rho^* \approx \pi$ by minimising the Kullback-Leibler divergence towards the target:
\begin{equation}
\label{eq:KL optimisation}
    \rho^* = \argmin_{\rho \in \mathcal{D}} \KL(\rho|\pi).
\end{equation}
While the accuracy of VI is limited by the expressivity of $\mathcal{D}$, the optimisation problem \eqref{eq:KL optimisation} can often be solved efficiently and at scale using modern (stochastic) gradient descent type algorithms \cite[Chapter 8]{Goodfellow-et-al-2016}.
\\\\
\emph{Markov Chain Monte Carlo} (MCMC) \cite{brooks2011handbook,robert2013monte} techniques, on the other hand, are asymptotically exact, being based on judiciously designed ergodic Markov processes $(X_t)_{t \ge 0}$ that admit $\pi$ as their invariant measure. The target is obtained as an appropriate limit of a long-time ergodic average:
\begin{equation}
\label{eq:ergodic limit}
    \pi = \lim_{T \rightarrow \infty} \frac{1}{T} \int_0^T \delta_{X_t} \, \mathrm{d}t.
\end{equation}
Accompanying convergence guarantees typically make inferences resting on MCMC more reliable than those based on VI. However, MCMC is challenging to parallelise and, furthermore, in high-dimensional settings it is often frustrated by slow convergence in \eqref{eq:ergodic limit} due to time correlations in $(X_t)_{t \ge 0}$.    
\\\\
Recently, there has been a growing interest in developing hybrid approaches that hold the promise of combining the advantages of MCMC and VI, see, for instance \cite{hoffman2017learning,maddison2017filtering,naesseth2018variational,ruiz2019contrastive,salimans2015markov}. 
Various attempts in this direction can be grouped into the so-called \emph{particle optimisation techniques} \cite{ambrogioni2018wasserstein,chen2017particle,chen2018unified,liu2019understanding} that posit carefully designed  dynamical schemes for an ensemble of particles $\bar{X} = (X^1,\ldots,X^N) \in (\mathbb{R}^d)^N$. From the VI-perspective, the variational family is then given by the empirical measures associated to the particles, $\mathcal{D} = \{\tfrac{1}{N}\sum^N_{i=1} \delta_{X^i}\}$, with the parameter set $\Phi$ corresponding to the positions of these particles. In terms of MCMC, the dynamics of $(\bar{X}_t)_{t \ge 0}$ can often be at least approximately thought of as a Markov process approaching an extended target $\bar{\pi}$ on $(\mathbb{R}^d)^N$ whose marginals coincide with $\pi$. 
\\

An appealing theoretical framework for analysing and constructing these particle-based methods is provided by the theory of gradient flows on probability distributions \cite{ambrosio2008gradient,otto2001geometry,otto2005eulerian}, connecting diffusions with $\mathrm{KL}$-optimisation problems of the form \eqref{eq:KL optimisation} on the grounds of differential geometric ideas. In this regard, the prime example (and also historically the first one where these concepts were layed out, see \cite{jordan1998variational}) is given by the overdamped Langevin dynamics \cite[Section 4.5]{pavliotis2014stochastic}, the associated Fokker-Planck equation of which takes the form of a gradient flow evolution driven by the $\mathrm{KL}$-divergence in the geometry induced by the quadratic Wasserstein distance. Recently, similar ideas have been pursued, replacing either the driving functional or the underlying geometry, see, for instance, \cite{arbel2019maximum,duncan2019geometry,garbuno2019interacting,garbuno2019affine,liu2017stein,reich2013ensemble,trillos2020bayesian}. 
\\

In statistical physics, gradient flow structures have been shown to play a major role in understanding the fluctuations of associated (stochastic) interacting particle systems \cite{mielke2016generalization,Onsager1931I, Onsager1953MachlupI} as described by the theory of large deviations. In this paper paper we utilise the correspondence between gradient flow structures and large-deviation functionals to shed some light on the connection between VI and MCMC in the context of a particular particle optimisation scheme, namely Stein variational gradient descent.  

\subsection{Stein Variational Gradient Descent}
\label{subsec:intro Stein}

Following the VI-paradigm, \emph{Stein variational gradient descent} (SVGD) was first derived in 
\cite{liu2016stein} from a minimising movement scheme for an ensemble of particles, seeking to iteratively solve the problem \eqref{eq:KL optimisation} for the corresponding empirical measure, while at the same time constraining the driving vector field to be chosen from within the unit ball of a reproducing kernel Hilbert space (RKHS)\footnote{Even though the $\mathrm{KL}$-divergence between the empirical measure and $\pi$ is not defined (or infinite), this statement can be made precise using the closely related \emph{kernelised Stein discrepancy} \cite{liu2016kernelized}.}. 
The method can be described by the following coupled system of ODEs, where $k: \mathbb{R}^d \times \mathbb{R}^d \rightarrow \mathbb{R}$ is a positive definite kernel of sufficient regularity\footnote{We refer to Section \ref{sec:notation} for precise assumptions.} and $\bar{X}_t = (X_t^1, \ldots X_t^N) \in (\mathbb{R}^d)^N$ denotes the ensemble of particles:
\begin{equation}
      \frac{\mathrm{d}X_t^i}{\mathrm{d}t} = \mfrac{1}{N} \sum_{j=1}^N \left( - k(X_t^i, X_t^j) \nabla V(X_t^j) + \nabla_{X_t^j} k(X_t^i, X_t^j) \right), \qquad i = 1,\ldots,N.
\label{eq:ODE}
\end{equation}

Crucial to this approach is the observation that the corresponding empirical measure 
\begin{equation}
\label{eq:empirical measure}
  \rho\super{N}_t := \frac{1}{N} \sum_{i=1}^N \delta_{X_t^i}
\end{equation}
converges to the target $\pi$ in an appropriate sense as both $N\rightarrow \infty$ and $t \rightarrow \infty$, see \cite{lu2019scaling} for rigorous statements.

In \cite{gallego2018stochastic}, the authors proposed to augment \eqref{eq:ODE} and obtained the interacting system of stochastic differential equations (SDEs) 
\begin{equation}
  \mathrm{d}X_t^i = \frac{1}{N} \sum_{j=1}^N \left[ - k(X_t^i, X_t^j) \nabla V(X_t^j) + \nabla_{X_t^j} k(X_t^i, X_t^j) \right] \mathrm{d}t + \sum_{j=1}^N \sqrt{2\mathcal{K}(\bar{X}_t)}_{ij} \, \mathrm{d}W_t^j,  \qquad i =1, \ldots, N,
  \label{eq:SDE}
\end{equation}
where the matrix-valued function $\mathcal{K}: (\mathbb{R}^{d})^N\rightarrow \mathbb{R}^{dN \times dN}$ consists of $N^2$ blocks of size $d \times d$, given 
by
$
  \mathcal{K}_{ij}(\bar{x}) = \frac{1}{N}k(x_i,x_j)I_{d\times d}
$, for $i,j \in \{1, \ldots, d\}$ and $\bar{x} = (x_1,\ldots,x_N)$.
Here, $(W^j_t)_{t \ge 0}$, $j=1,\ldots,N$ denotes a collection of $d$-dimensional standard Brownian motions, and $\sqrt{\mathcal{K}(\bar{X}_t)}$ refers to the matrix square root. The noise contribution $\sum_{j=1}^N \sqrt{2\mathcal{K}(\bar{X}_t)}_{ij} \, \mathrm{d}W_t^j$ has been designed so as to make the product measure $\bar{\pi}$ on $(\mathbb{R}^{d})^N$ with Lebesgue density
\begin{equation}
  \bar{\pi}(\bar{x}) = \pi(x_1) \cdot \ldots \cdot \pi(x_N), \qquad \qquad \bar{x} = (x_1,\ldots,x_N),
\end{equation}
invariant for the dynamics \eqref{eq:SDE}. In fact, under reasonable assumptions, \cite[Proposition 3]{duncan2019geometry} shows that \eqref{eq:SDE} is indeed ergodic with respect to $\bar{\pi}$, meaning that the associated empirical measure converges to $\pi$ as $t \rightarrow \infty$ (for instance, in total variation distance).
These observations show that the process $(\bar{X}_t)_{t \ge 0}$ solving \eqref{eq:SDE} can indeed be considered of MCMC-type, targeting $\bar{\pi}$. Indeed, \eqref{eq:SDE} can be cast in the framework of \cite{ma2015complete} as pointed out in \cite{gallego2018stochastic}. 

\subsection{A connection between MCMC and VI rooted in statistical physics}
\label{subsec:statphys}

One of the main topics in this article is the connection between the ODE \eqref{eq:ODE} and the SDE \eqref{eq:SDE}. First of all, the empirical measures associated to the solutions of  \eqref{eq:ODE} and \eqref{eq:SDE} become indistinguishable in the limit as $N \rightarrow \infty$, that is, the noise term $\sum_{j=1}^N \sqrt{2\mathcal{K}(\bar{X}_t)}_{ij} \, \mathrm{d}W_t^j$ becomes negligible. This claim can be substantiated in the sense that the \emph{Stein PDE} \cite{liu2017stein,lu2019scaling} 
\begin{equation}
\label{eq:Stein pde}
  \partial_t \rho_t(x) = \nabla_x \cdot \left( \rho_t(x) \int_{\mathbb{R}^d} \left[ k(x,y) \nabla V(y) - \nabla_y k(x,y)  \right]\,\rho_t(\mathrm{d}y) \right)
\end{equation}
describes the evolution of the empirical measure $\rho_t^{(N)}$ for both \eqref{eq:ODE} and \eqref{eq:SDE} in the mean field regime, that is, when $N \rightarrow \infty$.
To be more precise, for any fixed $N\in\mathbb{N}$, the empirical measure $\rho\super{N}_t$ associated to the ODE~\eqref{eq:ODE} satisfies \eqref{eq:Stein pde} in a weak sense, see \cite[Prop.~ 2.5]{lu2019scaling}, and, moreover, stability arguments show that this statement can be extended to the limit $N \rightarrow \infty$, see \cite[Theorem 2.7]{lu2019scaling}. Concerning the SDE \eqref{eq:SDE}, the additional noise term has been shown to be of order $\mathcal{O}(\tfrac{1}{N})$ in \cite[Proposition 3]{duncan2019geometry} and thus the corresponding empirical measure $\rho_t^{(N)}$ formally satisfies \eqref{eq:Stein pde} in the limit as $N\rightarrow \infty$. In this paper the latter convergence will be made more quantitative in terms of a corresponding large-deviation functional. 
\\

The Stein PDE~\eqref{eq:Stein pde} admits a gradient flow structure, described in~\cite{liu2017stein} and further analysed in \cite{duncan2019geometry}, that is, it can be written in the form $\partial_t \rho_t=-\Grad_{k}\KL(\rho_t)$, where $\KL$ is the Kullback-Leibner divergence or relative entropy towards $\pi$, and the gradient is with respect to a particular geometry determined by the kernel $k$; we shall make these terms more precise in Section~\ref{sec:GF}.
The gradient flow structure referred to above is not uniquely determined by \eqref{eq:Stein pde}; in fact the existence of one particular gradient flow structure implies that the PDE \eqref{eq:Stein pde} admits infinitely many gradient flow structures \cite{dietert2015characterisation}. For a particular example  see~ \cite{chewi2020svgd}, replacing the $\KL$- by the $\chi^2$-divergence. Our first main result shows that the $\KL$-gradient flow structure is naturally connected to the noise contribution in \eqref{eq:SDE}, bridging between the MCMC and VI viewpoints:
\begin{result}The gradient flow structure 
\begin{equation}
\partial_t \rho_t=-\Grad_{k}\KL(\rho_t)
\label{eq:GF structure}
\end{equation}
for the Stein PDE~\eqref{eq:Stein pde} (see Section~\eqref{sec:GF}) is compatible with the particular form of the noise in the SDE~\eqref{eq:SDE}.
\label{res:ldp and GF}
\end{result}
This statement will be made precise in Section~\ref{sec:connection}, resting on a reformulation of the Stein PDE \eqref{eq:Stein pde} in terms of a variational (in-)equality  (see Proposition \ref{prop:EDE}) and the large-deviation functional for the $N \rightarrow \infty$ limit associated to the SDE \eqref{eq:SDE}, see Theorem \ref{th:large deviations}. 
Intuitively, both the gradient flow scheme \eqref{eq:GF structure} and the large-deviation functional related to the noise structure in \eqref{eq:SDE} encode information that goes beyond what is described by the Stein PDE \eqref{eq:Stein pde}: The formulation \eqref{eq:GF structure} determines a specific non-unique `factorisation' of the right-hand side of \eqref{eq:Stein pde} into the geometric term $\mathrm{grad}_k$ and the driving functional $\mathrm{KL}$, while the SDE \eqref{eq:SDE} determines a non-unique\footnote{However, the noise contribution in \eqref{eq:SDE} is canonical in Bayesian inference as it ensures ergodicity with respect to the extended target $\bar{\pi}$.} stochastic augmentation of \eqref{eq:ODE}. The Informal Result \ref{res:ldp and GF} establishes a correspondence between those extensions of \eqref{eq:Stein pde} rooted in statistical physics; this general principle can be seen as a modern version of Onsager's reciprocity relation~\cite{mielke2014relation,Onsager1931I}.

\subsection{Speed of convergence, kernel choice and Stein-Fisher information}

From the practical perspective of minimising the computational cost, a central question is how to choose $k$ in such a way that the convergence $\rho_t^{(N)} \rightarrow \pi$ as $N\rightarrow \infty$ and $t \rightarrow \infty$ 
occurs `as rapidly as possible', that is, in such a way that $\rho_t\super{N}$ can provide a reasonable approximation of $\pi$ for $t$ and $N$ not too large. In \cite{duncan2019geometry}, the authors used convexity arguments along the geodesics induced by the Stein geometry for studying the $t\rightarrow \infty$ limit of the PDE \eqref{eq:Stein pde}, that is, for the study of the long-time behaviour in the many-particle regime.
\\

In the present paper, we complement those results, quantifying the speed of convergence for the random dynamics \eqref{eq:SDE} as $N\rightarrow \infty$ using the theory of large deviations (see Section \ref{sec:LD}). As a consequence of the Informal Result \ref{res:ldp and GF}, the relevant functional admits an elegant formulation in terms of the Stein geometry (see Theorem \ref{th:large deviations}). Although the methods presented in this paper concern the SDE system \eqref{eq:SDE},
these results allow us to gain some intuition into finite-particle effects for the deterministic system \eqref{eq:ODE} on a heuristic level  (see Section \ref{sec:comparing fisher}).
\\

In order to gain further insight and in particular to derive
practical guidelines for the choice of $k$, we next identify the leading order term in the large-deviation functional when $t$ is large, where limits 
are performed in the sense of Gamma convergence.
In this regime, both the many-particle as well as the long-time asymptotics turn out be closely related to the \emph{Stein-Fisher information}
\begin{subequations}
\label{eq:Fisher}
\begin{align}
\label{eq:Fisher1}
    I^{k}_\Stein(\rho) & = \int_{\mathbb{R}^d} \int_{\mathbb{R}^d} \nabla \frac{\mathrm{d}\rho}{\mathrm{d}\pi} (x) \cdot k(x,y) \nabla \frac{\mathrm{d} \rho}{\mathrm{d}\pi}(y) \, \pi(\mathrm{d}x) \pi(\mathrm{d}y)
    \\
    & =
    \label{eq:Fisher log}
    \int_{\mathbb{R}^d}
    \int_{\mathbb{R}^d}
    \left(\nabla \log \frac{\mathrm{d} \rho}{\mathrm{d}\pi} \right)(x) \cdot  k(x,y) \left(\nabla \log \frac{\mathrm{d} \rho}{\mathrm{d}\pi} \right)(y) \,
    \pi(\mathrm{d}x) \pi(\mathrm{d}y),
\end{align}
\end{subequations}
a quantity that has natural links with the cotangent space construction to be introduced in Section \ref{subsec:cotangent}. Let us also note that $I^k_{\mathrm{Stein}}(\rho)$ is known in other contexts as the \emph{kernelised Stein discrepancy} $\mathrm{KSD}(\rho|\pi)$ and has found various applications in scenarios where $\rho$ needs to be compared to an unnormalised\footnote{Indeed, \eqref{eq:Fisher log} shows that $I^k_{\mathrm{Stein}}(\rho)$ can be computed from $\pi = \tfrac{1}{Z}e^{-V}$ without knowing the potentially intractable normalisation constant $Z$.} distribution $\pi$, see \cite{chwialkowski2016kernel,fisher2020measure,gorham2017measuring}. In fact, the kernelised Stein discrepancy lies as the heart of the original derivation of SVGD, see \cite{liu2016stein}.
We summarise our findings in the following informal statement (to be explained and justified in Section \ref{sec:Fisher}).
\begin{result}
\label{res:informal}
The Stein-Fisher information~$I_{\mathrm{Stein}}^k$ controls the speed of convergence of the empirical measure associated to the SDE~\eqref{eq:SDE} in the regime when $N$ and $t$ are large.
As a consequence, letting $k_1,k_2:\mathbb{R}^d \times \mathbb{R}^d \rightarrow \mathbb{R}$ be two positive definite kernels with corresponding empirical measures $\rho_t^{(N),k_1}$ and $\rho_t^{(N),k_2}$ as defined in \eqref{eq:empirical measure} and \eqref{eq:SDE}, if
\begin{align}
\label{eq:Fisher comparison}
    I^{k_1}_{\Stein}(\rho) \ge I^{k_2}_{\Stein}(\rho), 
\end{align}
for all $\rho$ such that \eqref{eq:Fisher1} is well defined,
then the convergence of $\rho_t^{(N),k_1}$ towards $\pi$ as $N \rightarrow \infty$ and $t \rightarrow \infty$ is expected to be faster than the corresponding convergence of $\rho_t^{(N),k_2}$.
\end{result}

The preceding result applies when $N$ is large, but not infinite, hence taking a step towards understanding the finite-particle properties of SVGD.
Naturally, our two main results are strongly related. Indeed, the fact that the Stein-Fisher information~\eqref{eq:Fisher} controls the speed of convergence in both the $t \rightarrow \infty$ and $N\rightarrow \infty$ limits is ultimately a consequence of the compatibility between the gradient flow and noise structures expressed in the Informal Result \ref{res:ldp and GF}.
Let us state straight away that the comparison \eqref{eq:Fisher comparison} can be made on the basis of the reproducing kernel Hilbert spaces (RKHS) associated to $k_1$ and $k_2$. 
More precisely, we shall prove the following result (see Section \ref{sec:Fisher}).
\begin{proposition}
\label{prop:Fisher comparison}
Let $k_1,k_2:\mathbb{R}^d \times \mathbb{R}^d \rightarrow \mathbb{R}$ be two positive definite kernels satisfying Assumptions \ref{ass:k}, \ref{ass:bounded}, and \ref{ass:ISPD} below, 
and denote by $\mathcal{H}_{k_1}$ and $\mathcal{H}_{k_2}$ the corresponding reproducing kernel Hilbert spaces. Furthermore, assume that $V$ satisfies Assumption \ref{ass:V}.
Then the following are equivalent:
\begin{enumerate}
    \item
    \label{it:Fish comparison}
    \emph{Stein-Fisher comparison:} The inequality \eqref{eq:Fisher comparison} holds for all $\rho$ such that \eqref{eq:Fisher1} is well defined,
    \item
    \emph{Inclusion of RKHS-balls:}
    \label{it:RKHS balls}
    It holds that $\mathcal{H}_{k_2} \subset \mathcal{H}_{k_1}$ and
    \begin{align}
    \Vert \phi \Vert_{\mathcal{H}_{k_2}} \le
    \Vert \phi \Vert_{\mathcal{H}_{k_1}},
    &&\text{for all } \phi \in \mathcal{H}_{k_1}.
    \end{align}
\end{enumerate}
\end{proposition}
We refer the reader to Section \ref{sec:comparing fisher} for a proof, and to Section \ref{sec:examples} for an illustration of this result. Noting that the Stein-Fisher information coincides with the kernelised Stein discrepancy $\mathrm{KSD}(\rho|\pi)$, Proposition \ref{prop:Fisher comparison} might be of independent interest. 

\subsubsection*{Previous work}

Stein variational gradient descent in its original deterministic form \eqref{eq:ODE} was put forward in the seminal paper \cite{liu2016stein}. The stochastic variant \eqref{eq:SDE} was proposed in \cite{gallego2018stochastic} and shown to be ergodic in \cite{duncan2019geometry}. The fact that the Stein PDE \eqref{eq:Stein pde} admits a gradient flow structure was first observed in \cite{liu2017stein}; the corresponding Stein geometry was further developed in \cite{duncan2019geometry}, focusing on curvature and the long-time convergence properties of \eqref{eq:Stein pde}. This analysis revealed the important role played by the Stein-Fisher information \eqref{eq:Fisher} and the associated Stein log-Sobolev inequality. Based on this, the authors of \cite{korba2020non} developed nonasymptotic bounds in discrete time as well as  propagation of chaos results (the latter of which unfortunately are not uniform in time). We would also like to mention the work \cite{lu2019scaling} that rigorously establishes well-posednedness as well as convergence of the Stein PDE \eqref{eq:Stein pde}, and the work \cite{chewi2020svgd} that establishes an alternative gradient flow structure to the one considered in this paper.

\subsubsection*{Our contributions and outline of the article}

In this article we make the following contributions:

\begin{itemize}
    \item 
    We complement the geometric constructions from \cite{duncan2019geometry}, defining appropriate cotangent spaces and inner products. The Stein-Fisher information \eqref{eq:Fisher} (or kernelised Stein discrepancy) is shown to have a natural interpretation in terms of this framework.
    \item
    We compute the large-deviation functional associated to the mean field limit of the SDE \eqref{eq:SDE} and show that it can be expressed conveniently in terms of the tangent norm in the Stein geometry.
    \item
    On the basis of the obtained large-deviation rate functional, we connect the $\mathrm{KL}$-gradient flow structure in \eqref{eq:GF structure} with the noise structure in \eqref{eq:SDE}, providing a correspondence between the VI-type scheme \eqref{eq:ODE}  and the MCMC-type scheme \eqref{eq:SDE}.
    \item
    We identify the leading order term in the large-deviation functional in the regime where $t$ is large to obtain a direct relation to the Stein-Fisher information \eqref{eq:Fisher}.
    We argue that at a heuristic level, this result provides insight into finite-particle properties of SVGD.  
\end{itemize}

The article is organised as follows. In Section \ref{sec:notation} we introduce essential notation and state our basic assumptions. Furthermore, we provide an overview of the relevant background on reproducing kernel Hilbert spaces. In Section \ref{subsec:formal Riemann}, we review the geometric constructions from \cite{duncan2019geometry}. In Section \ref{subsec:cotangent}, we extend this work by defining the cotangent structure and establish its basic properties. Furthermore, we provide a reformulation of the Stein PDE \eqref{eq:Stein pde} in terms of a variational (in-)equality. In Section \ref{sec:LD} we derive the large-deviation rate functional for the mean field limit, leveraging the framework introduced in Section \ref{sec:GF}. In Section \ref{sec:connection}, we explain the connection between gradient flows and large deviations and make the Informal Result \ref{res:ldp and GF} precise. In Section \ref{sec:Fisher}, we identify the Stein-Fisher information as the leading order term in the large-deviation rate functional, provide a precise statement of the Informal Result \ref{res:informal}, and prove Proposition \ref{prop:Fisher comparison}. Furthermore, we provide a numerical example that illustrates our results. 
Finally, we conclude the paper in Section \ref{sec:examples} and briefly discuss directions for future work.

\section{Preliminaries}
\label{sec:notation}
In this section, we introduce essential notations and assumptions that are used throughout this article. In addition, we briefly point out a few key results in the theory of reproducing kernel Hilbert spaces. For textbook accounts, the reader is referred to \cite{berlinet2011reproducing,saitoh2016theory,scholkopf2018learning,steinwart2008support}.

\subsection{Notation and general assumptions}

In order to ensure that both the target measure $\pi$ in \eqref{eq:target} as well as the dynamics \eqref{eq:ODE} and \eqref{eq:SDE} are well-defined, we assume that the given potential satisfies the following:
\begin{assumption}[Assumptions on $V$]
\label{ass:V}
The potential $V$ is continuously differentiable,
    $V \in C^1(\mathbb{R}^d)$, and $e^{-V}$ is integrable, $\int_{\mathbb{R}^d} e^{-V} \, \mathrm{d}x < \infty$.
\end{assumption}
The set of probability measures on $\mathbb{R}^d$ will be denoted by $\mathcal{P}(\mathbb{R}^d)$. For any $\rho \in \mathcal{P}(\mathbb{R}^d)$, the Hilbert space of $\rho$-square-integrable functions will be denoted by $L^2(\rho)$, with scalar product $\langle \phi,\psi \rangle_{L^2(\rho)} = \int_{\mathbb{R}^d} \phi \psi \, \mathrm{d}\rho$ and associated norm $\Vert \phi \Vert^2_{L^2(\rho)} = \langle \phi,\phi \rangle_{L^2(\rho)}$. Often, we will work with the following subset of probability measures, 
\begin{equation}
\label{eq:definition M}
  M := \left\{ \rho \in \mathcal{P}(\mathbb{R}^d) : \quad \rho \,\,\text{admits a smooth and strictly positive density with respect to the Lebesgue measure} 
\right\}.
\end{equation}
We later formally turn this set into a Riemannian manifold with an extended geodesic distance (allowing the value $\infty$) that depends on the choice of the kernel.

\subsection{Assumptions on kernels}

Throughout this paper, we work with one or more kernels that are always assumed to satisfy the following:

\begin{assumption}
\label{ass:k}
The kernel $k:\mathbb{R}^d \times \mathbb{R}^d \rightarrow \mathbb{R}$ is assumed to be symmetric, continuous, and continuously differentiable off the diagonal, that is, $k \in C^1(\mathbb{R}^d \times \mathbb{R}^d \setminus \{(x,y) \in \mathbb{R}^{2d}: \, x = y  \})$. Furthermore, $k$ is assumed to be positive definite, that is, for all $n \in \mathbb{N}$, $\alpha_1,\ldots,\alpha_n \in \mathbb{R}$ and $x_1, \ldots, x_n \in \mathbb{R}^d$ it holds that $\sum_{i,j \ge 1}^n \alpha_i \alpha_j k(x_i,x_j) \ge 0$.
\end{assumption}

\begin{assumption}
\label{ass:bounded}
The kernel $k$ is bounded.
\end{assumption}

\begin{assumption}
\cite{fukumizu2009kernel,sriperumbudur2010hilbert}
\label{ass:ISPD}
The kernel $k$ is \emph{integrally strictly positive definite} (ISPD), that is,
\begin{equation}
\int_{\mathbb{R}^d} \int_{\mathbb{R}^d} k(x,y) \, \rho(\mathrm{d}x) \, \rho(\mathrm{d}y) >0,
\end{equation}
for all signed Borel measures $\rho$ that are not the zero measure.
\end{assumption}

Let us comment on the foregoing assumptions. While Assumption \ref{ass:k} is fundamental (in that it is required for the construction of associated reproducing kernel Hilbert spaces (RKHS) as well as for defining all the terms in \eqref{eq:ODE} and \eqref{eq:SDE}), Assumptions \ref{ass:bounded} and \ref{ass:ISPD} are made in this paper for technical convenience. Indeed, the set-up in \cite{duncan2019geometry} encompasses unbounded kernels (but does require the weaker integrability condition $\int_{\mathbb{R}^d} k(x,x) \, \mathrm{d}\rho(x) < \infty$ for measures $\rho$ under consideration). Non-ISPD kernels have been considered in \cite{liu2018stein}, for instance, and could be included in our framework with more technical effort. Note that the ISPD Assumption~\ref{ass:ISPD} is a strengthened version of the positive definiteness in Assumption~\ref{ass:k}.

Examples of kernels satisfying Assumptions \ref{ass:k}, \ref{ass:bounded} and \ref{ass:ISPD} are given by the parametric family 
$k_{p,\sigma}:\mathbb{R}^d \times \mathbb{R}^d \rightarrow \mathbb{R}$, defined via
\begin{equation}
	\label{eq:p kernel}
	k_{p,\sigma}(x,y) = \exp\left(-\frac{\vert x - y \vert^p}{\sigma^p}\right),
\end{equation}
where  $p \in (0,2]$ is a smoothness parameter, and $\sigma > 0$ is called the kernel width (see \cite[Lemma 42]{duncan2019geometry}). Further examples are provided by the family of \emph{Mat{\'e}rn kernels} whose reproducing kernel Hilbert spaces coincide with the classical Sobolev spaces $W^{m,2}(\mathbb{R}^d)$ whenever $m$ and $d$ are such that $W^{m,2}(\mathbb{R}^d) \subset C(\mathbb{R}^d)$, see \cite[Section 1.3]{saitoh2016theory}.

\subsection{Reproducing kernel Hilbert spaces}

Given a positive definite kernel $k$,
we denote by $(\mathcal{H}_k, \langle \cdot, \cdot \rangle_{\mathcal{H}_k})$ the corresponding reproducing kernel Hilbert space (RKHS), see \cite[Section 4]{steinwart2008support}, and by $\Vert \cdot \Vert^2_{\mathcal{H}_k} = \langle \cdot, \cdot \rangle_{\mathcal{H}_k}$ the associated norm. This Hilbert space is characterised by the conditions that $k(x,\cdot) \in \mathcal{H}_k$ as well as $\langle f, k(x,\cdot) \rangle_{\mathcal{H}_k} = f(x)$, for all $x \in \mathbb{R}^d$ and $f \in \mathcal{H}_k$. If $\rho \in \mathcal{P}(\mathbb{R}^d)$ is a probability measure with full support, then Assumption \ref{ass:bounded} ensures that $\mathcal{H}_k \subset L^2(\rho)$, where moreover the natural inclusion is continuous, see \cite[Theorem 4.26]{steinwart2008support}, and Assumption \ref{ass:ISPD} guarantees that $\mathcal{H}_k \subset L^2(\rho)$ is dense, see \cite[Theorem 7]{sriperumbudur2010hilbert} and \cite[Theorem 4.26i)]{steinwart2008support}.

In order to characterise the norm $\Vert \cdot \Vert_{\mathcal{H}_k}$ more explicitly, it is helpful to introduce the operators $\T_{k,\rho}:L^2(\rho) \rightarrow L^2(\rho)$
\begin{equation}
\label{eq:T_krho}
  (\T_{k,\rho} \phi)(x) := \int_{\mathbb{R}^d} k(x,y) \phi(y) \rho(\mathrm{d}y), \qquad \phi \in L^2(\rho).
\end{equation}

We gather a number of properties of this operator that will be useful later on.
\begin{proposition} 
\label{prop:H properties}
For all $\rho\in M$,
\begin{enumerate}[label=\emph{(\alph*)}]
    \item $\T_{k,\rho} L^2(\rho)\subset\Hcal_k$, and 
    $\T_{k,\rho}:L^2(\rho) \rightarrow \mathcal{H}_k$ is the adjoint of the inclusion $\mathcal{H}_k \hookrightarrow L^2(\rho)$, that is
    \begin{equation}
    \label{eq:kernel trick}
    \langle \T_{k,\rho} \phi, \psi\rangle_{\mathcal{H}_k} = \langle \phi, \psi\rangle_{L^2(\rho)}, \qquad \phi \in L^2(\rho),\,\,\psi \in \mathcal{H}_k.
    \end{equation}
    \label{propit:T range}
    \item $\T_{k,\rho}$ is compact, self-adjoint and positive semi-definite on $L^2(\rho)$, 
    \label{propit:T compact}
    \item $\T_{k,\rho}$ is injective. \label{propit:T inj}
\end{enumerate}
\label{prop:T properties}
\end{proposition}
\begin{proof}
For \ref{propit:T range} and \ref{propit:T compact}, see \cite{steinwart2008support}, Theorems 4.26 and 4.27, respectively. For \ref{propit:T inj}, notice that $\T_{k,\rho} \phi = 0$, $\phi \in C_c^{\infty}(\mathbb{R}^d)$ implies $\phi = 0$ by integrating against $\phi \rho$ and using Assumption \ref{ass:ISPD}.
\end{proof}
\begin{remark}
The identity \eqref{eq:kernel trick} is a key calculational tool throughout the proofs in Section \ref{sec:GF} and can formally be viewed as a consequence of the defining identity $\langle k(x,\cdot),f\rangle_{\mathcal{H}_k} = f(x)$ after commuting integration and $\langle \cdot,\cdot \rangle_{\mathcal{H}_k}$.
\end{remark}

The scalar product in $\mathcal{H}_k$ can now be written in the form
\begin{equation}
\label{eq:T isometry}
    \langle f, g \rangle_{\mathcal{H}_k} = \langle \T_{k,\rho}^{-1/2} f, \T_{k,\rho}^{-1/2} g \rangle_{L^2(\rho)}, \qquad f,g \in \mathcal{H}_k, 
\end{equation}
where $\T_{k,\rho}^{-1/2}$ may be defined via the spectral theorem \cite[Chapter VII]{reed2012methods}. For instance, if $(e_i)_{i=1}^\infty \subset L^2(\rho)$ is an orthonormal eigenbasis of $\T_{k,\rho}$ (that is, $\langle e_i, e_j \rangle_{L^2(\rho)} = \delta_{ij}$ and $\T_{k,\rho} e_i = \lambda_i e_i$), then for $f= \sum_i f_i e_i$ and $g = \sum_i g_i e_i$ we have that
\begin{equation*}
  \langle f, g \rangle_{\mathcal{H}_k} = \sum_{i=1}^\infty \frac{1}{\lambda_i} f_i g_i, 
\end{equation*}
see \cite[Section 4.5]{steinwart2008support}.

Derived from $\mathcal{H}_k$ and $L^2(\rho)$, we will frequently make use of the corresponding spaces of vector fields $\mathcal{H}_k^d$ and $(L^2(\rho))^d$, defined through 
\begin{equation*}
\mathcal{H}^d_k = \underbrace{\mathcal{H}_k \otimes \ldots \otimes \mathcal{H}_k}_{d \, \text{times}} \qquad \text{and} \qquad (L^2(\rho))^d = \underbrace{L^2(\rho) \otimes \ldots \otimes L^2(\rho)}_{d \, \text{times}}.
\end{equation*}
In other words, $\mathcal{H}_k^d$ and $(L^2(\rho))^d$ consist of vector fields $v = (v_1,\ldots,v_d)$, with $v_i \in \mathcal{H}_k$ or $v_i \in L^2(\rho)$, respectively, with scalar products given by
\begin{align*}
    \langle v,w \rangle_{\mathcal{H}_k^d} = \sum_{i=1}^d \langle v_i,w_i \rangle_{\mathcal{H}_k}, \qquad v_i, w_i \in \mathcal{H}_k, 
        &&\text{and}&&
    \langle v, w \rangle_{(L^2(\rho))^d} = \sum_{i=1}^d \langle v_i,w_i\rangle_{L^2(\rho)}, \qquad v_i,w_i \in L^2(\rho).
\end{align*}
The operators $\T_{k,\rho}$ defined in \eqref{eq:T_krho} straightforwardly extend to the space $(L^2(\rho))^d$, interpreting \eqref{eq:T_krho} componentwise. Similarly, Proposition \ref{prop:H properties} as well as the identity \eqref{eq:T isometry} remain valid with the obvious modifications.
Finally, we will need the following result in the spirit of the usual Helmholtz-decomposition \cite{schweizer2018friedrichs}. 
\begin{proposition}[Helmholtz decomposition for RKHS]
\label{prop:helmholtz}
Let $\rho \in M$ and define the space of divergence-free vector fields
\begin{equation}
L^2_{\mathrm{div}}(\rho) = \left\{ v \in (L^2(\rho))^d: \quad \langle v, \nabla \phi \rangle_{(L^2(\rho))^d} = 0, \quad \text{for all } \phi \in C_c^{\infty}(\mathbb{R}^d)\right\}.
\end{equation}
Then $\mathcal{H}_k^d$ admits the following $\langle \cdot, \cdot \rangle_{\mathcal{H}^d_k}$-orthogonal decomposition,
\begin{equation}
    \mathcal{H}_k^d = \left( L^2_{\mathrm{div}}(\rho) \cap \mathcal{H}_k^d \right) \oplus \overline{\T_{k,\rho} \nabla C_c^\infty(\mathbb{R}^d)}^{\mathcal{H}_k^d}.
\end{equation}
\end{proposition}
\begin{proof}
We refer to \cite[Lemma 45]{duncan2019geometry}.
\end{proof}

\section{The Stein PDE as a gradient flow}
\label{sec:GF}
In this section we recall and further analyse the \emph{Stein geometry} that allows us to formally write the Stein PDE \eqref{eq:Stein pde} as a gradient flow on probability distributions, as first observed in \cite{liu2017stein}. Subsection~\ref{subsec:formal Riemann} will mostly be a review of the Stein geometry as developed in \cite{duncan2019geometry}; in Subsection~\ref{subsec:cotangent} we complement the construction from \cite{duncan2019geometry} by defining appropriate cotangent spaces endowed with inner products; those turn out to be closely related to the Stein-Fisher information \eqref{eq:Fisher}. The duality between tangent and cotangent spaces gives rise to a variational reformulation of the Stein PDE \eqref{eq:Stein pde} in Proposition \ref{prop:EDE} that  will be instrumental in linking the large-deviation statement in Section \ref{sec:LD} to the gradient flow structure of the mean field limit. Analysing the Stein PDE \eqref{eq:Stein pde} using the geometric picture outlined in this section is very much inspired by the works of Otto and coworkers on the Fokker-Planck equation and its relation to the quadratic Wasserstein distance (see \cite{jordan1998variational,otto1998dynamics,otto2001geometry,otto2000generalization,otto2005eulerian} as well as the further developments in \cite{ambrosio2008gradient,gigli2012second} and \cite{daneri2008eulerian}). For a direct comparison between the Stein geometry and the Wasserstein geometry we refer the reader to \cite[Appendix A]{duncan2019geometry}.

Anticipating the constructions to follow in the remainder of this section, let us already in intuitive terms lay out the connections between the original idea from \cite{liu2016stein} and the central geometric concepts of the Stein geometry. In \cite{liu2016stein}, the authors construct the ODE \eqref{eq:ODE} as the continuous-time limit of a gradient descent scheme. More precisely, they consider an ensemble of particles, represented by the empirical measure $\rho^{(N)}$, and design a minimising movement scheme that aims at minimising the $\mathrm{KL}$-divergence between $\rho^{(N)}$ and the target $\pi$. The associated velocity field is constrained to be chosen from within the RKHS $\mathcal{H}_k^d$ and obtained from a variational principle that involves the corresponding RKHS-norm. As observed in \cite{liu2017stein} and further developed in \cite{duncan2019geometry}, this construction principle is linked to the observation that \eqref{eq:Stein pde} can be cast in the form
\begin{equation}
\label{eq:gradient flow}
    \partial_t \rho = -\mathbb{K}_\rho \frac{\delta \mathrm{KL}}{\delta \rho} =: -(\mathrm{grad}_k \mathrm{KL})(\rho),
\end{equation}
where $\mathrm{KL}$ denotes the Kullback-Leibler\footnote{For notational convenience later on, we adopt the notation $\mathrm{KL}(\rho) := \mathrm{KL}(\rho | \pi)$, suppressing the dependence on $\pi$.} divergence (or relative entropy) between the current distribution $\rho_t$ and the target $\pi$,
\begin{equation}
\label{eq:KL}
  \KL(\rho) = \int_{\mathbb{R}^d}\!\log \left(\frac{\mathrm{d}\rho}{\mathrm{d}\pi}\right)\mathrm{d}\rho = \int_{\mathbb{R}^d} V \, \mathrm{d}\rho + \int_{\mathbb{R}^d} \log \rho \, \mathrm{d}\rho + \log Z, \qquad \qquad \rho \in M,
\end{equation}
and $\KK_\rho$ is a positive definite `Onsager' operator that we introduce in~\eqref{eq:Onsager operator}. This operator defines the Stein-gradient $\mathrm{grad}_k := \mathbb{K}_\rho \frac{\delta}{\delta \rho}$, formalises the minimising movement scheme from \cite{liu2016stein} and can be seen to be induced by an appropriate definition of the tangent spaces $T_\rho M$ and corresponding (formal) Riemannian metric. The Onsager operators $\KK_\rho$ translate between the tangent and cotangent spaces defined below; indeed we have $\mathrm{grad}_k \mathrm{KL} \in T_\rho M$ and $\frac{\delta \mathrm{KL}}{\delta \rho} \in T_\rho^* M$, at least formally.

From a statistical perspective, the term $\int_{\mathbb{R}^d} V \, \mathrm{d}\rho$ in \eqref{eq:KL} measures the fit to the data, the entropic term ${\int_{\mathbb{R}^d}} \log \rho \, \mathrm{d}\rho$ encodes regularisation, and the normalisation constant $Z$ represents the Bayesian evidence, useful in the context of model selection (see, for instance \cite{liutkus2019sliced} and \cite[Section 6.7]{gelman2013bayesian}). 

\subsection{Formal Riemannian structure and associated gradient}
\label{subsec:formal Riemann}

In what follows, we formally equip the set $M$ defined in \eqref{eq:definition M} with the structure of a Riemannian manifold, following \cite[Section 4]{duncan2019geometry}, where the reader is referred to for further details. To start with, recall the operators $\T_{k,\rho} $ from \eqref{eq:T_krho}, that can be extended to self-adjoint, nonnegative definite, and compact operators on $L^2(\rho)$, see \cite[Section 4.3]{steinwart2008support}.
By abuse of notation, we will often apply $\T_{k,\rho}$ to vector fields in $(L^2(\rho))^d$, in which case \eqref{eq:T_krho} is to be understood componentwise. The operators $\T_{k,\rho}$ are used to define the tangent space construction in the Stein geometry:
\begin{definition}[Tangent spaces and Riemannian metric] 
	\label{def:Riemann geometry}
See \cite[Definition 5]{duncan2019geometry}.
	For $\rho \in M$, we define the \emph{tangent space}
	\footnote{
	    $\mathcal{D}'$ denotes the usual space of Schwartz distributions, as the dual of $\mathcal{D}(\R^d):=C_c^\infty(\R^d)$ equipped with the Schwartz topology, see\cite{duistermaat2010distributions}. Moreover, we say that $\xi + \nabla \cdot (\rho v) = 0$ holds in the sense of distributions if $\langle \xi, \phi \rangle  - \int_{\mathbb{R}^d}  \nabla \phi \cdot v \, \mathrm{d}\rho = 0$ 
        for all $\phi \in C_c^\infty(\mathbb{R}^d)$, where $\langle \cdot, \cdot \rangle$ denotes the standard duality relation between $\mathcal{D}'(\mathbb{R}^d)$ and $C_c^{\infty}(\mathbb{R}^d)$.
	}
\begin{subequations}
\nonumber
	\label{eq:tangent spaces}
	\begin{align}
	T_{\rho} M:= \Bigg\{ \xi \in \mathcal{D}'(\mathbb{R}^d):\,\,    \text{there exists } v \in \overline{\T_{k,\rho} \nabla C_c^\infty(\mathbb{R}^d)}^{\mathcal{H}_k^d} \,\, \text{such that}\, 
	 \xi + \nabla \cdot (\rho v) = 0 \, \text{in the sense of distributions}
	\Bigg\}
	\end{align}
\end{subequations}
and the \emph{Riemannian metric} $\langle \cdot, \cdot \rangle_{T_\rho M}: T_\rho M \times T_\rho M \rightarrow \mathbb{R}$ by
\begin{equation} 
\label{eq:Stein metric}
    \langle \xi, \chi \rangle_{T_\rho M} := \langle u, v \rangle_{\mathcal{H}_k^d}, 
\end{equation}
where $\xi + \nabla \cdot(\rho u ) = 0$ and $\chi + \nabla \cdot(\rho v ) = 0$, as well as $u,v \in \overline{\T_{k,\rho} \nabla C_c^\infty(\mathbb{R}^d)}^{\mathcal{H}_k^d}$.
\end{definition}
A few remarks concerning Definition \ref{def:Riemann geometry} are in order. First of all, the spaces $\overline{\T_{k,\rho} \nabla C_c^\infty(\mathbb{R}^d)}^{\mathcal{H}_k^d}$ mimic the spaces $\overline{ \nabla C_c^\infty(\mathbb{R}^d)}^{L^2(\rho)}$ common in the Wasserstein setting, see for example~\cite[Sec.~8.4]{ambrosio2008gradient}. Similar to that scenario, to each $\xi\in T_\rho M$ there exists a unique $u \in \overline{\T_{k,\rho} \nabla C_c^\infty(\mathbb{R}^d)}^{\mathcal{H}_k^d}$ with $\xi + \nabla \cdot(\rho u ) = 0$, so that~\eqref{eq:Stein metric} is justified. This fact can be traced back to the Helmholtz decomposition in the RKHS setting, see Proposition \ref{prop:helmholtz}. Furthermore, $\overline{\T_{k,\rho} \nabla C_c^\infty(\mathbb{R}^d)}^{\mathcal{H}_k^d}$ may also be recognised as the set of vector fields which are permissible in minimising movement schemes such as those devised in the original paper \cite{liu2017stein}. Therefore, at an intuitive level, $T_\rho M$ is the space of derivatives $\partial_t \rho$, where $\rho$ is a curve obtained by continuous-time limits of these schemes. We refer the reader to \cite[Lemma 7]{duncan2019geometry}, showing that indeed $T_\rho M$ is a well-defined Hilbert space, for all $\rho \in M$. 
\\\\
The following lemma shows that $T_\rho M$ can indeed be considered the tangent space. Before we come to this result, we recall that the functional derivative of a suitable functional $\mathcal{F}:M \rightarrow \mathbb{R}$ is defined via
\begin{equation}
\label{eq:functional derivative}
  \int_{\mathbb{R}^d} \frac{\delta \mathcal{F}}{\delta{\rho}}(\rho)(x) \phi(x) \, \mathrm{d}x := \frac{\mathrm{d}}{\mathrm{d}\varepsilon} \Big\vert_{\varepsilon = 0} \mathcal{F}(\rho + \varepsilon \phi),
\end{equation}
for $\phi \in C_c^\infty(\mathbb{R}^d)$ with $\int_{\mathbb{R}^d} \phi \, \mathrm{d}x = 0$.
The functional derivative of the Kullback-Leibner divergence~\eqref{eq:KL} can be computed for $\rho\in M$:
\begin{equation}
\label{eq:KL func derivative}
	\frac{\delta \mathrm{KL}}{\delta \rho}(\rho)(x) = \log \rho(x) + V(x),
\end{equation}
see, for instance, \cite[Chapter 15]{V2009}.
\\\\
Finally, we are able to connect the geometric construction from Definition \ref{def:Riemann geometry} with the Stein PDE \eqref{eq:Stein pde}:
\begin{lemma}[Stein gradient]
\label{lem:gradient}
See \cite[Lemma 9 and Corollary 11]{duncan2019geometry}.
Let $\rho \in M$ and $\mathcal{F}: M \rightarrow \mathbb{R}$ be such that the functional derivative $\frac{\delta \mathcal{F}}{\delta \rho}(\rho)$ is well-defined and continuously differentiable. Moreover assume that $\T_{k,\rho} \nabla \frac{\delta \mathcal{F}}{\delta \rho}(\rho) \in \overline{\T_{k,\rho} \nabla C_c^\infty(\mathbb{R}^d)}^{\mathcal{H}_k^d}$.  Then the Riemannian gradient associated to $(T_\rho M, \langle \cdot, \cdot \rangle_{T_\rho M})$ is given by
\begin{equation}
\label{eq:gradient}
(\mathrm{grad}_k \mathcal{F})(\rho) = - \nabla \cdot \left( \rho \, \T_{k,\rho} \nabla \mfrac{\delta \mathcal{F}}{\delta \rho} (\rho) \right).
\end{equation}
Using \eqref{eq:KL func derivative}, it follows that the gradient flow formulation \eqref{eq:gradient flow} and the Stein PDE \eqref{eq:Stein pde} coincide.
\end{lemma}

\subsection{Cotangent spaces, Onsager operators, duality\\ and the energy-dissipation (in-)equality}
\label{subsec:cotangent}

In this section we expand the ideas of~\cite{duncan2019geometry} and define the cotangent spaces $T_\rho^*M$, their duality relationship with the tangent spaces through the Onsager operators $\mathbb{K}_\rho$, and establish their basic properties. In order to construct the cotangent spaces, we begin by defining the corresponding inner products for sufficiently regular test functions.

\begin{definition}[Dual inner product]
For $\rho \in M$, we define the \emph{dual inner product}
\begin{equation}
\label{eq:dual inner product}
\langle \phi, \psi \rangle_{T_\rho^* M} = \int_{\mathbb{R}^d} \int_{\mathbb{R}^d} \nabla \phi(x) \cdot k(x,y) \nabla \psi(y) \rho(\mathrm{d}x)  \rho(\mathrm{d}y), \qquad \phi,\psi \in C_c^{\infty}(\mathbb{R}^d),
\end{equation}
as well as the \emph{Onsager operator}
\begin{subequations}
\label{eq:Onsager}
\begin{align}
\mathbb{K}_\rho: \quad C_c^{\infty}(\mathbb{R}^d) & \rightarrow T_\rho M
\\
\phi & \mapsto - \nabla \cdot (\rho \T_{k,\rho} \nabla \phi).
\label{eq:Onsager operator}
\end{align}
\end{subequations}
\end{definition}

\begin{remark}
\label{rem:Onsager}
Combining the definition \eqref{eq:Onsager} with \eqref{eq:gradient}, we see that $\mathrm{grad}_k = \mathbb{K}_\rho \frac{\delta}{\delta \rho}$. In differential geometric terms, the functional derivative $\frac{\delta}{\delta \rho}$ takes the role of the exterior derivative \cite{lee2006riemannian}, while the Onsager operator corresponds to the musical isomorphisms (`raising' the index in the language of theoretical physics). The latter concept will be made more explicit in Proposition \ref{prop:duality} below. Note also that \eqref{eq:Onsager operator} is similar to the Wasserstein setting where $\KK_\rho\phi=-\nabla \cdot (\rho \nabla \phi)$.
\end{remark}

The next lemma is a prelude to Definition \ref{def:cotangent spaces}, in particular showing that the inner product $\langle \cdot,\cdot\rangle_{T_\rho^* M}$ is nondegenerate.
\begin{lemma}
Let $k$ satisfy Assumptions \ref{ass:k}, \ref{ass:bounded} and  \ref{ass:ISPD}. Then $(C_c^{\infty}(\mathbb{R}^d, \langle \cdot, \cdot \rangle_{T_\rho^* M})$ is a pre-Hilbert\footnote{A pre-Hilbert (or inner-product) space satisfies the usual axioms of a Hilbert space, except for completeness. That is, it does not necessarily contain the limit points of all Cauchy sequences, see \cite[Section 3.1]{kreyszig1978introductory}.} space over $\mathbb{R}$.
\end{lemma}
\begin{proof}
The bilinearity of $\langle \cdot, \cdot \rangle_{T_\rho^* M})$ is immediate from the definition.  For $\phi \in C_c^\infty(\mathbb{R}^d)$, Assumption \ref{ass:ISPD} implies that $\langle \phi,\phi \rangle_{T_\rho^*M} = 0$ if and only if $\phi = 0$.
\end{proof}
The cotangent spaces can now be defined as follows:
\begin{definition}[Cotangent spaces]
\label{def:cotangent spaces}
For $\rho \in M$, we define the cotangent spaces $T_\rho^*M$ to be the completions\footnote{Any pre-Hilbert space can be upgraded to a Hilbert space, intuitively by considering all limit points. For a rigorous survey of the completion construction see \cite[Section 1.6, Theorem 3.2-3]{kreyszig1978introductory}.} of $(C_c^{\infty}(\mathbb{R}^d), \langle \cdot, \cdot \rangle_{T_\rho^* M})$.
\end{definition}

\begin{remark}
\label{rem:Fisher cotangent}
On a practical level, the completion construction extends  the definition  \eqref{eq:dual inner product} to functions $\phi$ such that $\langle \phi,\phi \rangle_{T_\rho^* M}$ can be defined\footnote{Note that this does not require $\phi$ to be differentiable; indeed \eqref{eq:dual inner product} can be extended to nondifferentiable $\phi$ and $\psi$ using integration by parts.} and $\langle \phi,\phi \rangle_{T_\rho^* M} < \infty$. In particular, if $\rho \in M$ and $I^k_{\mathrm{Stein}}(\rho) < \infty$ with $I^k_{\mathrm{Stein}}$ defined as in \eqref{eq:Fisher}, then $\tfrac{\rho}{\pi} \in T_\pi^*M$, and
\begin{equation}
    I^k_{\mathrm{Stein}}(\rho) = \left\Vert\mfrac{\rho}{\pi}\right\Vert^2_{T_\pi^*M} =  \left\Vert\mfrac{\delta \mathrm{KL}}{\delta \rho}\right\Vert^2_{T_\rho^*M}.
    \label{eq:Fisher equalities}
\end{equation}
We will revisit this identity in Section \ref{sec:Fisher}.
\end{remark}

The next result shows that the Onsager operators $\mathbb{K}_\rho$ naturally translate between the tangent and cotangent spaces, substantiating Remark \ref{rem:Onsager}:

\begin{proposition}[Duality]
\label{prop:duality}
For any $\rho \in M$,
the Onsager operator $\mathbb{K}_\rho$ extends to an isometric isomorphism between the Hilbert spaces $T_\rho^*M$ and $T_\rho M$. That is, the extension (denoted by the same symbol) satisfies
    \begin{equation}
    \label{eq:Onsager isometry}
    \langle \mathbb{K}_\rho \phi, \mathbb{K}_\rho \psi \rangle_{T_\rho M} = \langle \phi, \psi \rangle_{T_\rho^* M},
    \end{equation}
    for all $\phi,\psi \in T_\rho^*M$.
\end{proposition}

\begin{proof}
For $\phi,\psi \in C_c^\infty(\mathbb{R}^d)$, we have that 
\begin{equation}
\label{eq:isometry}
    \langle \mathbb{K}_\rho \phi, \mathbb{K}_\rho \psi \rangle_{T_\rho M} = \langle \T_{k,\rho} \nabla \phi, \T_{k,\rho} \nabla \psi \rangle_{\mathcal{H}_k^d} = \langle \T_{k,\rho} \nabla \phi, \nabla \psi \rangle_{(L^2(\rho))^d} = \langle \phi,\psi \rangle_{T_\rho^* M}.
\end{equation}
Here, the first identity follows from the definition \eqref{eq:Stein metric} and Proposition~\ref{prop:T properties}\ref{propit:T range}, while the second identity is implied by the adjoint relation \eqref{eq:kernel trick}. The third identity is a direct consequence of the definition \eqref{eq:dual inner product}. From \eqref{eq:isometry}, we see that  $\mathbb{K}_\rho$ is a linear isometry from $C_c^{\infty}(\mathbb{R}^d)$ to $T_\rho M$, and hence can be uniquely extended to an isometry $\widehat{\mathbb{K}}_\rho$ on the completion $T_\rho^*M$ (see \cite[Theorem I.7]{reed2012methods}). Being an isometry, it is clear that $\widehat{\mathbb{K}}_\rho$ is injective. It remains to show that $\widehat{\mathbb{K}}_\rho$ is surjective. To this end, it is sufficient to prove that $\mathbb{K}_\rho(C_c^{\infty}(\mathbb{R}^d))$ is dense in $T_\rho M$. For this, let us assume to the contrary that $\mathbb{K}_\rho(C_c^{\infty}(\mathbb{R}^d))$ is not dense. Then there exists $\chi \in T_\rho M$ with $\chi \neq 0$ such that $\langle \chi, \mathbb{K}_\rho \phi \rangle_{T_\rho M} = 0$, for all $\phi \in C_c^\infty(\mathbb{R}^d)$. By Definition \ref{def:Riemann geometry}, there exists $v \in \mathcal{H}_k^d$ such that $\chi + \nabla \cdot (\rho v)  = 0$ in the sense of distributions, as well as a sequence $(\psi_n) \subset C_c^{\infty}(\mathbb{R}^d)$ such that $\T_{k,\rho} \nabla \psi_n \rightarrow v$ in $\mathcal{H}_k^d$. We then see that
\begin{equation}
0 = \langle \chi, \mathbb{K}_\rho \psi_n \rangle_{T_\rho M} = - \langle v, \T_{k,\rho} \nabla \psi_n \rangle_{\mathcal{H}_k^d} \rightarrow - \Vert v \Vert_{\mathcal{H}_k^d}^2,
\end{equation}
implying that $v = 0$. From the second statement in \cite[Lemma 7]{duncan2019geometry}, implied by Proposition \ref{prop:helmholtz}, it then follows that $\chi  = 0$, contradicting the assumption $\chi \neq 0$ from before, and hence concluding the proof.
\end{proof}

We can leverage the correspondence between $T_\rho M$ and $T_\rho^*M$ through $\mathbb{K}_\rho$ provided by Proposition \ref{prop:duality} to set up an associated duality relation. This duality is natural in that it coincides with the duality between $\mathcal{D}'(\mathbb{R}^d)$ and $C_c^{\infty}(\mathbb{R}^d)$ whenever both are defined:

\begin{corollary}
\label{cor:duality}
For any $\rho \in M$, we can define the duality relation 
\begin{equation}
\label{eq:Onsager duality}
    \tensor[_{T_\rho^*M}]{\langle \phi,\xi \rangle}{_{T_\rho M}} := \langle \mathbb{K}_\rho \phi, \xi \rangle_{T_\rho M}, \qquad \phi \in T_\rho^* M, \,\,\xi \in T_\rho M.
\end{equation}
In particular, $T_\rho^*M$ is a representation of the dual of $T_\rho M$.
If $\phi \in C_c^\infty(\mathbb{R}^d) \subset T_\rho^* M$, then we have
\begin{equation}
\label{eq:natural duality}
    \tensor[_{T_\rho^*M}]{\langle \phi,\xi \rangle}{_{T_\rho M}} =  
    \tensor[_{\mathcal{D}'(\mathbb{R}^d)}]{\langle \xi, \phi\rangle}{_{C_c^{\infty}(\mathbb{R}^d)}} = 
\int_{\mathbb{R}^d} v \cdot \nabla \phi \, \mathrm{d}\rho,
\end{equation}
where $\xi + \nabla \cdot(\rho v) = 0$ and $v \in \overline{\T_{k,\rho} \nabla C_c^{\infty}(\mathbb{R}^d)}^{H_k^d}$.
\end{corollary}

\begin{proof}
By Proposition \ref{prop:duality} and the Riesz representation theorem, \eqref{eq:Onsager duality} establishes a one-to-one correspondence between the topological dual of $T_\rho M$ and $T_\rho^*M$. The second identity in \eqref{eq:natural duality} is satisfied by definition, see Remark \ref{rem:Onsager}. To obtain the first identity, consider a sequence $(\xi_n) \subset T_\rho^*M$ with $\xi_n \rightarrow \xi$, such that there exists a sequence $(\psi_n) \subset C_c^{\infty}(\mathbb{R}^d)$ satisfying $\xi_n + \nabla \cdot(\rho \T_{k,\rho}\nabla \psi_n) = 0$ in the sense of distributions. Then we have
\begin{subequations}
\begin{align}
    \tensor[_{\mathcal{D}'(\mathbb{R}^d)}]{\langle \xi_n, \phi \rangle}{_{C_c^{\infty}(\mathbb{R}^d)}} & = \int_{\mathbb{R}^d} (\T_{k,\rho} \nabla \psi_n) \cdot \nabla \phi \, \mathrm{d}\rho = \langle \psi_n, \phi \rangle_{T_\rho^* M} = \langle \mathbb{K}_\rho \psi_n, \mathbb{K}_\rho \phi \rangle_{T_\rho M}
    \\
    & = \tensor[_{T_\rho^*M}]{\langle \phi, \mathbb{K}_\rho \psi_n \rangle}{_{T_\rho M}} = \tensor[_{T_\rho^* M}]{\langle \phi, \xi_n \rangle}{_{T_\rho M}},
\end{align}
where the first inequality is a consequence of the definition of $\tensor[_{\mathcal{D}'(\mathbb{R}^d) }]{\langle \cdot, \cdot \rangle}{_{C_c^{\infty}(\mathbb{R}^d)}}$, the second equality follows from \eqref{eq:dual inner product}, the third equality follows from \eqref{eq:Onsager isometry}, and the fourth equality follows from \eqref{eq:Onsager duality}. Finally, we obtain \eqref{eq:natural duality} by passing to the limit, noting that all operations are continuous.
\end{subequations}
\end{proof}

As an consequence of this duality and the Banach-Alaoglu Theorem, we obtain compactness of the (sub-)level sets of the Stein-Fisher information, relevant later in the proof of Theorem \ref{th:large time Gamma}.
\begin{corollary} For any $C>0$, the sets $\{\rho\in M: I^k_\Stein(\rho)\leq C\}$ are pre-compact in the topology characterised by the convergence:
\begin{align}
    \rho^\epsilon \wsconv{\epsilon\to0} \rho 
    \quad:\Leftrightarrow\quad
    \tensor[_{T_\pi^*M}]{\big\langle \mfrac{\rho^\epsilon}{\pi},\xi \big\rangle}{_{T_\pi M}}
        \to
    \tensor[_{T_\pi^*M}]{\big\langle \mfrac{\rho}{\pi},\xi \big\rangle}{_{T_\pi M}} \qquad\text{for all } \xi\in T_\pi M.
\label{eq:ws topology}
\end{align}
\label{cor:compact level sets}
\end{corollary}

We next provide a reformulation of the Stein PDE \eqref{eq:Stein pde} in terms of an energy-dissipation (in-)equality (see \cite[Chapter 11]{ambrosio2008gradient}), using the framework developed in this section:

\begin{proposition}[Energy-dissipation equality]
\label{prop:EDE}
For $T>0$, let $\rho:[0,T] \rightarrow M$ be a curve such that $t \mapsto \mathrm{KL}(\rho_t)$ is differentiable and for all $t \in [0,T]$,
\begin{align}
    \partial_t \rho_t \in T_{\rho_t}M,
    &&
    \frac{\delta \mathrm{KL}}{\delta \rho} (\rho_t) \in T^*_{\rho_t}M,
    &&\text{and}&&
    \frac{\mathrm{d}}{\mathrm{d}t} \mathrm{KL}(\rho_t) = \tensor[_{T_{\rho_t}^* M}]{\left \langle  \frac{\delta \mathrm{KL}}{\delta \rho} (\rho_t),\partial_t \rho_t \right \rangle}{_{T_{\rho_t} M}}.
\label{eq:conditions EDI}
\end{align}
Then the following statements are equivalent:
\begin{align}
    &\text{the Stein PDE~\eqref{eq:Stein pde} holds for all } t\in\lbrack0,T\rbrack,\notag\\
    &\qquad\iff\notag\\
    &\partial_t\rho_t=-\KK_{\rho_t}\frac{\delta\KL}{\delta\rho}(\rho_t) \qquad\text{for all } t\in\lbrack0,T\rbrack,\notag\\
    &\qquad\iff\notag\\
    &\mathrm{KL}(\rho_T) - \mathrm{KL}(\rho_0) + \int_0^T \left( \mfrac12\left\Vert \partial_t \rho_t \right\Vert^2_{T_{\rho_t} M} + \mfrac12\left\Vert \mfrac{\delta \mathrm{KL}}{\delta \rho} (\rho_t) \right\Vert_{T_{\rho_t}^* M}^2 \right) \mathrm{d}t = 0.
\label{eq:EDI}
\end{align}
Moreover, for any curve satisfying~\eqref{eq:conditions EDI}, the left-hand side of \eqref{eq:EDI} is nonnegative.
\end{proposition}

\begin{remark}
The assumptions \eqref{eq:conditions EDI} are made for convenience, and we refer to \cite{ambrosio2008gradient} for generalisations. Note that the chain rule, i.e. the last condition of~\eqref{eq:conditions EDI}, is expected to hold at a formal level, combining \eqref{eq:functional derivative} and \eqref{eq:natural duality}. Since \eqref{eq:EDI} is always non-negative, the proposition continues to hold if `$=$' is replaced by `$\le$'; analogues of \eqref{eq:EDI} are therefore often called \emph{energy-dissipation inequalities} in the literature.
\end{remark}

\begin{proof}
Take any curve satisfying~\eqref{eq:conditions EDI}. The statement follows immediately by applying~\eqref{eq:conditions EDI} and completing the square:
\begin{multline*}
    \mathrm{KL}(\rho_T) - \mathrm{KL}(\rho_0) + \frac{1}{2} \int_0^T\!\left( \left\Vert \partial_t \rho_t \right\Vert^2_{T_{\rho_t} M} + \left\Vert \mfrac{\delta \mathrm{KL}}{\delta \rho} (\rho_t) \right\Vert_{T_{\rho_t}^* M}^2 \right) \mathrm{d}t\\
    =
    \int_0^T\!\left(
        \tensor[_{T_{\rho_t}^*M}]{\Big{\langle}\mfrac{\delta \mathrm{KL}}{\delta \rho} (\rho_t),\partial_t \rho_t \Big{\rangle}}{_{T_{\rho_t} M}}
        + 
        \frac{1}{2} \left\Vert \partial_t \rho_t \right\Vert^2_{T_{\rho_t} M} 
        + 
        \frac12 \left\Vert\mfrac{\delta \mathrm{KL}}{\delta \rho} (\rho_t) \right\Vert_{T^*_{\rho_t} M}^2 \right) \mathrm{d}t\\
    \stackrel{\eqref{eq:Onsager isometry},\eqref{eq:Onsager duality}}{=}
        \frac12\int_0^T\!\left\lVert\ \partial\rho_t + \KK_{\rho_t}\mfrac{\delta\KL}{\delta\rho}(\rho_t)\right\rVert_{T_{\rho_t}M}^2\,\dd t.
\end{multline*}
\end{proof}

We conclude this section with a remark on the relationship between the functional-analytic frameworks associated to the Stein and Wasserstein geometries. For this, we recall that in the Wasserstein geometry, the tangent and cotangent spaces are given by the Sobolev spaces $H^{-1}(\rho)$ and $H^1(\rho)$, respectively, see \cite{mielke2016generalization}.

\begin{lemma}[Comparison with the Wasserstein setting] We have
\label{rem:sandwich}
\begin{equation}
H^1(\rho) \hookrightarrow T_{\rho}^*M \qquad \text{and} \qquad T_\rho M \hookrightarrow H^{-1}(\rho),
\end{equation}
where $\hookrightarrow$ denotes containment with continuous inclusion.
\end{lemma}
\begin{remark}
The fact that the tangent spaces  $T_\rho M$ in the Stein geometry are contained in the tangent spaces for the Wasserstein geometry is ultimately due to the fact that the movement of the particles is restricted to vector fields belonging to reproducing kernel Hilbert spaces in SVGD.
\end{remark}
\begin{proof}
For the first statement, it it sufficient to show that there exists a constant $C>0$ such that $\Vert \phi \Vert_{T_\rho^* M} \le C \Vert \phi \Vert_{H^1(\rho)}$, for all $\phi \in C_c^{\infty}(\mathbb{R}^d)$. This follows immediately from 
\begin{equation}
\langle \phi,\phi \rangle_{T_\rho^*M} = \langle \T_{k,\rho} \nabla \phi,\nabla \phi \rangle_{(L^2(\rho))^d} \le \Vert \T_{k,\rho} \Vert_{(L^2(\rho))^d\rightarrow (L^2(\rho))^d} \Vert \nabla \phi \Vert_{(L^2(\rho))^d}, \quad \phi \in C_c^{\infty}(\mathbb{R}^d), 
\end{equation}
noting that $\T_{k,\rho}$ is bounded on $(L^2(\rho))^d$ by Proposition \ref{prop:H properties} and therefore $\Vert \T_{k,\rho} \Vert_{(L^2(\rho))^d} < \infty$. The second statement follows immediately by the duality established in Corollary \ref{cor:duality}, see \cite[Theorem 4.10]{rudin2006functional}.
\end{proof}

\section{Large deviations corresponding to the mean field limit}
\label{sec:LD}

In this section we introduce and derive the large-deviation principle for the empirical measure $\rho\super{N}_t$ associated to the SDE~\eqref{eq:SDE} as $N\to\infty$. The derivation will partly be  formal via a standard tilting technique; rigorous results for similar stochastic systems can be found in the classic works \cite{DawsonGaertner1987} and \cite[Ch.~13.3]{FengKurtz2006}.

As mentioned in Subsection~\ref{subsec:intro Stein}, the (random) path $\rho\super{N}:=(\rho\super{N}_t)_{t\in\lbrack0,T\rbrack}$ converges weakly as $N\to\infty$ to the solution $\rho\super{\infty}$ of the Stein PDE~\eqref{eq:Stein pde}. This means that for any continuity set $\A$ of paths~\cite[Th.~2.1]{Billingsley1999},
\begin{equation*}
  \PP(\rho\super{N}\in\A) \xrightarrow{N\to\infty} \delta_{\rho\super{\infty}}(\A),
\end{equation*}
that is, the probability vanishes for any atypical path $\rho\neq\rho\super{\infty}$. The large-deviation principle quantifies the exponential rate of this convergence:
\begin{equation}
  \PP(\rho\super{N}\in\A) \stackrel{N\to\infty}{\sim} \exp\big({\textstyle-N\inf_{\rho\in\A} \I_{\lbrack0,T\rbrack}(\rho)}\big),
\label{eq:general large-deviation principle}
\end{equation}
where the \emph{rate functional} $\I_{\lbrack0,T\rbrack}$ satisfies $\I_{\lbrack0,T\rbrack}(\rho\super{\infty})=0$ and $\I_{\lbrack0,T\rbrack}(\rho)>0$ for any path $\rho\neq\rho\super{\infty}$. In other words, the magnitude of $\I_{[0,T]}(\rho)$ quantifies the `unlikeliness' of the particular path $\rho$ as a deviation from $\rho^{(\infty)}$, in the exponential scaling indicated above. The infimum on the right-hand side appears because the process will follow the least unlikely path with overwhelming probability; for the precise definition of the large-deviation principle we refer to~\cite{DemboZeitouni1998}.

\begin{remark}
Throughout this paper we shall formally derive large-deviation principles~\eqref{eq:general large-deviation principle} for small balls $\Bcal_\epsilon(\rho)$ around an unlikely event $\rho$, i.e.
\begin{equation*}
  \PP(\rho\super{N}\in \Bcal_\epsilon(\rho)) \stackrel{N\to\infty}{\sim} \exp\big({\textstyle-N \I_{\lbrack0,T\rbrack}(\rho)}\big).
\end{equation*}
This is a common proof technique for both the large-deviation lower and upper bounds in the rigorous definition,  see ~\cite[Sec.~1.2]{DemboZeitouni1998}.
\end{remark}

\begin{remark}
For brevity we shall largely ignore the role of the initial condition $\rho_0$.  Implicitly we will always assume that the initial positions $X_0^i$ of all particles are chosen deterministically, in such a way that $\rho\super{N}_0$ converges weakly to some given $\rho_0$. Theorem \ref{th:large time Gamma} below shows that the leading order contribution as $T \rightarrow \infty$ is independent of $\rho_0$.
\end{remark}

Our main result provides an expression for the rate functional $\I_{[0,T]}$ in terms of the $T_\rho M$-norm introduced in Definition \ref{def:Riemann geometry}. We postpone a discussion of its interpretation until Sections \ref{sec:Fisher} and \ref{sec:examples}.

\begin{theorem}[Large-deviation principle, formal]
	\label{th:large deviations}
	The path $(\rho\super{N}_t)_{t\in\lbrack0,T\rbrack}$ of the empirical measure~\eqref{eq:empirical measure} associated to the SDE~\eqref{eq:SDE} satisfies a large-deviation principle~\eqref{eq:general large-deviation principle} with rate functional $\I_{\lbrack0,T\rbrack}$,  given by
	\begin{equation}
	\I_{\lbrack0,T\rbrack}(\rho) = \frac{1}{4} \int_0^T \left\Vert 
	  \partial_t \rho_t-\nabla_x \cdot \left( \rho_t \int_{\mathbb{R}^d} \left[ k(\cdot,y) \nabla V(y) - \nabla_y k(\cdot,y)  \right]\rho_t(\mathrm{d} y) \right)
	\right\Vert^2_{T_{\rho_t} M} \, \mathrm{d}t,
	\label{eq:LD I}
	\end{equation}
	for paths $\rho$ satisfying \eqref{eq:conditions EDI}.
\end{theorem}
We note that the expression~\eqref{eq:LD I} can be extended to arbitrary paths $\rho:[0,T]\to\Pcal(\R^d)$, possibly taking the value infinity; for brevity we will focus on sufficiently regular paths in the sense of \eqref{eq:conditions EDI}.

In the remainder of this section, we outline the proof of Theorem \ref{th:large deviations}. The key idea is to tilt the underlying probability measure using Girsanov transformations so that the atypical path $\rho$ becomes the typical one for the new, tilted measure. The very same technique is common in importance sampling for diffusions, see \cite{hartmann2012efficient,tzen2019theoretical} and \cite[Section 2.2]{nusken2020solving}, and sequential Monte Carlo methods \cite{del2004feynman,doucet2001introduction,reich2018data}, used to simulate the occurrence of rare (=atypical) events.
\\\\
The calculation of the large-deviation rate functional requires the construction of the `exponential martingale'\footnote{The exponential martingale is the right-hand side of~\eqref{eq:Girsanov} applied to the random process $\rho\super{N}_t$.}, for which it will be helpful to know the generator\footnote{Since the process $\rho_t^{(N)}$ takes (random) values in $\mathcal{P}(\mathbb{R}^d)$, its generator acts on functionals $F:\mathcal{P}(\mathbb{R}^d) \rightarrow \mathbb{R}$ of sufficient regularity. For background on measure-valued stochastic processes we refer the reader to \cite{Dawson1993}.} of the process $\rho_t^{(N)}$ explicitly.

\begin{lemma}
For each $N\in\Nb$, the generator of the Markov process $\rho^{(N)}_t$ defined by \eqref{eq:SDE} and \eqref{eq:empirical measure} is:
\begin{align}
	(\Qcal\super{N} F) (\rho) &= \iint_{\R^d\times\R^d}\!\big\lbrack - k(x,y) \nabla V(y) + \nabla_{y} k(x,y) \big\rbrack \cdot\nabla_x\!\left(\frac{\delta F}{\delta \rho}(\rho) \right)\! (x)\, \rho(\mathrm{d}x)\, \rho(\mathrm{d}y) \notag\\
	 &\hspace{2.3cm} + \mfrac{1}{N}\iint_{\R^d\times\R^d}\! k(x,y) \nabla_x\cdot\nabla_y\!\left( \frac{\delta^2 F }{\delta \rho^2}(\rho)\right)\!(x,y)\,\rho(\mathrm{d}x)\,\rho(\mathrm{d}y) \notag\\
	 &\hspace{4.8cm} + \mfrac{1}{N}\int_{\R^d}\!k(x,x) \Delta \!\left( \frac{\delta F}{\delta \rho} (\rho)\right)\!(x) \, \rho(\mathrm{d}x),
\label{eq:generator}
\end{align}
\label{lem:generator}
where $F:\mathcal{P}(\mathbb{R}^d) \rightarrow 
\mathbb{R}$ is a test function of sufficient regularity.
\end{lemma}
The proof follows a standard calculation involving It{\^o}'s formula that we postpone to the appendix.

The following result shows that the process can be perturbed or \emph{tilted} by adding an additional, time-dependent drift such that the Radon-Nikodym derivative is explicit.

\begin{lemma}[Girsanov transformation] Let $\PP_{[0,T]}\super{N}$ be the law of the empirical measure process $(\rho\super{N}_t)_{t\in\lbrack0,T\rbrack}$ associated to the   
SDE~\eqref{eq:SDE}, fix a test function
$G:[0,T]\times \mathcal{P}(\mathbb{R}^d)$ of sufficient regularity
\footnote{A convenient class of test functions is given by $G\in C_b^2\big(0,T;C_b^2(\Pcal(\R^d))\big)$, where the derivative in $\mathcal{P}(\mathbb{R}^d)$ is understood in the sense of \eqref{eq:functional derivative}. 
See \cite[Theorem 2.2.1]{ustunel2013transformation} for the general Novikov condition (in finite dimensions) and \cite{PalmowskiRolski2002} for an even more general condition.}
, and define the tilted measure $\PP_{[0,T]}\super{N,G}$ through the Radon-Nikodym derivative,
\begin{equation}
    \frac{d\PP_{[0,T]}\super{N,G}}{d\PP_{[0,T]}\super{N}}(\rho)
    =
    \exp\Big(
        NG_T(\rho_T)-NG_0(\rho_0) - N\int_0^T\!(\partial_t G_t)(\rho_t)\,\dd t - N\int_0^T\!(\Ham\super{N}G_t)(\rho_t)\,\dd t
    \Big),
\label{eq:Girsanov}
\end{equation}
where the operator $\mathcal{H}^{(N)}$ is defined as
\begin{equation}
    (\Ham\super{N} G)(\rho):=\frac1N e^{-NG(\rho)}\big(\Qcal\super{N} e^{NG(\rho)}\big)(\rho).
\label{eq:nonlinear semigroup}
\end{equation}
Then $\PP_{[0,T]}\super{N,G}$ is the law of the (time-inhomogeneous) Markov process with generator
\begin{equation}
    (\Qcal\super{N,G}_t F)(\rho):=(\Qcal\super{N}F)(\rho) 
    + 
    \iint_{\R^d\times\R^d}\! k(x,y) \nabla_x\cdot\nabla_y\,\!
    \Big(  
      \frac{\delta F}{\delta\rho}(\rho)(x)\frac{\delta G_t}{\delta\rho}(\rho)(y)
      +
      \frac{\delta G_t}{\delta\rho}(\rho)(x)\frac{\delta F}{\delta\rho}(\rho)(y)
    \Big)\!\,\rho(\mathrm{d}x)\,\rho(\mathrm{d}y).
\label{eq:tilted generator}
\end{equation}
\label{lem:Girsanov}
\end{lemma}
\begin{remark}
The limit of the operator $\Ham^{(N)}$ coincides with the generator of the nonlinear Nisio semigroup in the framework of \cite{FengKurtz2006}.
\end{remark}
\begin{proof}
The statement follows from the general Girsanov transformation formula~\cite[Th.~4.2]{PalmowskiRolski2002}, adapted to allow for time-dependent test functions $G$, see also \cite[Th.~A1.7.3]{Kipnis1999} and \cite[Sec.~8.6.1.1]{FengKurtz2006}:
\begin{equation*}
    (\Qcal\super{N,G}_t F)(\rho)= e^{-NG_t(\rho)}(\Qcal\super{N}F e^{NG_t(\rho)})(\rho) - e^{-NG_t(\rho)}F(\rho)(\Qcal\super{N} e^{NG_t(\rho)})(\rho).
\end{equation*}
\end{proof}

In the following we explicitly calculate~\eqref{eq:nonlinear semigroup} and pass to the limit as $N\to\infty$.

\begin{lemma} For any test function
$G:[0,T]\times \mathcal{P}(\mathbb{R}^d)$ of sufficient regularity (as above), we have
\begin{align*}
    (\Ham\super{N}G)(\rho)\xrightarrow{N\to\infty}\Ham\Big(\rho,\mfrac{\delta G}{\delta\rho}(\rho)\Big)
    :=
    - \left\langle \mfrac{\delta G}{\delta\rho}(\rho),\mfrac{\delta \KL}{\delta \rho}(\rho)\right\rangle_{T_\rho^*M} + \left\Vert \mfrac{\delta G}{\delta\rho}(\rho)\right\Vert_{T_\rho^*M}^2,
\end{align*}
where $\langle \cdot, \cdot \rangle_{T_\rho^*M}$ is defined in \eqref{eq:dual inner product}. 
\label{lem:limit Hamiltonian}
\end{lemma}
We postpone this calculation to the appendix.

\begin{lemma}[Mean-field limit, formal] Fix a test function $G:[0,T]\times \mathcal{P}(\mathbb{R}^d)$ of sufficient regularity (as above) and let $\rho\super{N,G}_t$ be the process with generator~\eqref{eq:tilted generator}. Then $\rho\super{N,G}\to\rho\super{\infty,G}$, which weakly solves the ``tilted Stein PDE'':
\begin{equation}
  \partial_t \rho_t(x) = \nabla_x \cdot \left( \rho_t(x) \int_{\mathbb{R}^d}\!\left[ k(x,y) \nabla V(y) - \nabla_y k(x,y) - 2k(x,y) \nabla_y \mfrac{\delta G_t}{\delta\rho}(\rho_t)(y) \right]\,\rho_t(\mathrm{d} y) \right).
\label{eq:tilted Stein pde}
\end{equation}

\label{lem:mean-field limit}
\end{lemma}
\begin{proof}[Proof sketch]
Clearly the generator~\eqref{eq:tilted generator} converges pointwise in $\rho$ to
\begin{align*}
    (\Qcal\super{\infty,G}_t F)(\rho):=&
    \iint_{\R^d\times\R^d}\!\big\lbrack - k(x,y) \nabla V(y) + \nabla_{y} k(x,y) \big\rbrack \cdot\nabla_x\!\left(\mfrac{\delta F}{\delta \rho}(\rho) \right)\! (x)\, \rho(\dd x)\, \rho(\mathrm{d}y)\\
    &+ 
    \iint_{\R^d\times\R^d}\! k(x,y) \nabla_x\cdot\nabla_y\!
    \Big(  
      \mfrac{\delta F}{\delta\rho}(\rho)(x)\mfrac{\delta G_t}{\delta\rho}(\rho)(y)
      +
      \mfrac{\delta G_t}{\delta\rho}(\rho)(x)\mfrac{\delta F}{\delta\rho}(\rho)(y)
    \Big)\!\,\rho(\mathrm{d}x)\,\rho(\dd y).
\end{align*}
Hence formally by \cite[Th.~2.12]{Liggett1985}, the process $\rho\super{N,G}_t$ converges to some (a priori stochastic) process $\rho\super{\infty,G}_t$. It remains to show that this process satisfies~\eqref{eq:tilted Stein pde} and is thus deterministic.

Let $P\super{\infty,G}_t$ be the time marginal of the path measure $\PP\super{\infty,G}_{[0,T]}$ and let us make the ansatz that it is indeed deterministic: $P\super{\infty,G}_t=\delta_{\rho_t}$ for some (by assumption) sufficiently regular path $\rho_t(\dd x)=\rho_t(x)\,\dd x$. Then using the Chapman-Kolmogorov forward equation,
\begin{align*}
    \int_{\R^d}\! \mfrac{\delta F}{\delta\rho_t}(\rho_t)(x)\,\partial_t\rho(\dd x)
    =
    \frac{\dd}{\dd t} F(\rho_t)
    =
    \frac{\dd}{\dd t}\int\!F(\rho)\,P_t(\dd \rho)
    &=
    \int\!(\Qcal\super{\infty,G}_t F)(\rho)\,P_t(\dd\rho)\\
    &= 
    (\Qcal\super{\infty,G}_t F)(\rho_t)
    =
    \int_{\R^d}\! \mfrac{\delta F}{\delta\rho}(\rho_t)(x)\,B\super{G_t}(\rho_t(x))\,\dd x,
\end{align*}
where $B\super{G_t}(\rho_t(x))$ is the right-hand side of \eqref{eq:tilted generator}. Since this equation holds for arbitrary (sufficiently regular) test functions $F$, it follows that $\partial_t\rho_t=B\super{G_t}(\rho_t)$ weakly, and so the ansatz is justified for $\rho_t=\rho\super{\infty,G}_t$.
\end{proof}

\begin{remark}
  In the notation introduced in Section~\ref{sec:GF}, the tilted Stein PDE~\eqref{eq:tilted Stein pde} becomes
  \begin{equation*}
      \partial_t\rho_t=-\KK_{\rho_t}\frac{\delta\KL}{\delta\rho}(\rho_t) + 2\KK_{\rho_t}\frac{\delta G_t}{\delta\rho}(\rho_t).
  \end{equation*}
\end{remark}

We finally have all the ingredients to prove the main result of this section.
\begin{proof}[Proof sketch of Theorem \ref{th:large deviations}]
To simplify, we only derive the rate functional for an arbitrary path $\rho=(\rho_t)_{t\in\lbrack0,T\rbrack}$ satisfying~\eqref{eq:conditions EDI}. Corresponding to this path, let $\xi_t\in T_{\rho_t}^*M$ be maximal in
\begin{equation*}
   \tensor[_{T_{\rho_t}^*M}]{\langle \xi_t,\partial_t\rho_t\rangle}{_{T_{\rho_t}M}} - \Ham(\rho_t,\xi_t)
    =
    \tensor[_{T_{\rho_t}^*M}]{\left\langle \xi_t, \partial_t\rho_t + \KK_{\rho_t}\mfrac{\delta \KL}{\delta\rho}(\rho_t)\right\rangle}{_{T_{\rho_t}M}} +  \tensor[_{T_{\rho_t}^*M}]{\langle\xi_t,\KK_{\rho_t}\xi_t\rangle}{_{T_{\rho_t}M}},
\end{equation*}
pointwise in $t\in[0,T]$, where the brackets are defined in Corollary \ref{cor:duality}, and $\mathcal{H}$ is the limit obtained in Lemma~\ref{lem:limit Hamiltonian}. Again for simplicity we shall assume that this maximiser exists, and in fact $\xi_t\in C_c^\infty(\R^d)\subset T_{\rho_t}^*M$.

Upon differentiation with respect to $\xi_t$ we recover the tilted Stein PDE~\eqref{eq:tilted Stein pde} with $G_t(\rho):=\tensor[_{\mathcal{D'}(\R^d)}]{\langle\rho_t,\xi_t\rangle}{_{C_c^\infty(\R^d)}}$ so that $(\delta G_t/\delta\rho)(\rho_t)=\xi_t$. 

Thus for this particular choice, Lemma~\ref{lem:mean-field limit} shows that the tilted process converges to the path we picked in the beginning of the proof, i.e.
\begin{equation}
    \PP_{[0,T]}\super{N,G}\to\delta_{\rho}.
\label{eq:inproof mean-field limit}
\end{equation}

Now pick an arbitrary small ball $\Bcal_\epsilon(\rho)$ (in Skorokhod space) around the path $\rho$. By Lemma~\ref{lem:Girsanov} we may change the measure and write, for small $\epsilon>0$:
\begin{align}
    \frac1N\log\PP_{[0,T]}\super{N}(\Bcal_\epsilon(\rho))
    =
    \frac1N\log\int_{\Bcal_\epsilon(\rho)} \!\frac{d\PP_{\lbrack0,T\rbrack}\super{N}}{d\PP_{\lbrack0,T\rbrack}\super{N,G}}(\hat \rho)\,d\PP_{\lbrack0,T\rbrack}\super{N,G}(\dd \hat \rho)
    \approx \frac1N\log\frac{d\PP_{\lbrack0,T\rbrack}\super{N}}{d\PP_{\lbrack0,T\rbrack}\super{N,G}}(\rho) + \frac1N\log\PP_{[0,T]}\super{N,G}(\Bcal_\epsilon(\rho)).
\label{eq:tilt application}
\end{align}
By~\eqref{eq:inproof mean-field limit} the last term vanishes, and so by Lemmas~\ref{lem:Girsanov} and~\ref{lem:limit Hamiltonian} we find as $N\to\infty$ and small $\epsilon>0$,
\begin{align*}
    \frac1N\log\PP_{[0,T]}\super{N}(\Bcal_\epsilon(\rho))
    &\to
    - \tensor[_{\mathcal{D'}(\R^d)}]{\langle \rho_T,\xi_T\rangle}{_{C_c^\infty(\R^d)}}
    + \tensor[_{\mathcal{D'}(\R^d)}]{\langle \rho_0,\xi_0\rangle}{_{C_c^\infty(\R^d)}} 
    - \int_0^T\!\tensor[_{\mathcal{D'}(\R^d)}]{\langle\rho_t,\partial_t\xi_t\rangle}{_{C_c^\infty(\R^d)}}\,\mathrm{d}t - \int_0^T\!\Ham(\rho_t,\xi_t)\,\dd t\\
    &= -\sup_{\hat\xi:(0,T)\to T_{\rho}^* M} \int_0^T\!\big\lbrack \tensor[_{T_{\rho_t}^*M}]{\langle \hat\xi_t,\partial_t\rho_t\rangle}{_{T_{\rho_t}M}} - \Ham(\rho_t,\hat\xi_t)\big\rbrack\,\dd t\\
    &= -\sup_{\hat\xi:(0,T)\to T_{\rho}^* M} \int_0^T\!\left\lbrack \tensor[_{T_{\rho_t}^*M}]{\left\langle \hat\xi_t,\partial_t\rho_t+\KK_{\rho_t}\mfrac{\delta \KL}{\delta \rho}(\rho_t)\right\rangle}{_{T_{\rho_t}M}} - \left\Vert \xi_t\right\Vert_{T_\rho^*M}^2\right\rbrack\,\dd t\\
    &=-\I_{\lbrack0,T\rbrack}(\rho),
\end{align*}
where the supremum over tiltings $\hat\xi$ appears due to the definition of $\xi$.

In general, the perturbation functions $\xi$ and $G$ do not have sufficient regularity to apply Lemma~\ref{lem:Girsanov}, if they exist at all, so one typically needs technically demanding approximation arguments to make this into a rigorous argument, see for example \cite[Ch. 13]{FengKurtz2006} and \cite{DawsonGaertner1987}. 
\end{proof}

From the large-deviation result in  Theorem~\ref{th:large deviations} and the contraction principle~\cite[Th.~4.2.1]{DemboZeitouni1998} we immediately obtain the large-deviation principle for the ergodic limit~\eqref{eq:ergodic limit}. The ensuing rate functional will be further analysed in Section \ref{sec:Fisher}.

\begin{corollary}
\label{cor:fdd} Fix $T>0$ and $\rho_0\in M$. Let $(\rho\super{N}_t)_{t\in\lbrack0,T\rbrack}$ be the path of the empirical measure associated to the SDE~\eqref{eq:SDE}, and let
\begin{equation*}
    \bar\rho\super{N}_T:=\mfrac1T\int_0^T\!\rho\super{N}_t\,\mathrm{d}t
\end{equation*}
be the ergodic average thereof. Then $\bar\rho\super{N}_T$ satisfies a large-deviation principle as $N\to\infty$, i.e.
\begin{equation*}
  \PP(\bar{\rho}\super{N}_T\in\A) \stackrel{N\to\infty}{\sim} \exp\big({\textstyle-N\inf_{\bar{\rho}\in\A} \bar\I_T(\bar{\rho})}\big),
\end{equation*}
with rate functional:
\begin{equation}
    \bar\I_T(\bar{\rho})
    :=
    \inf_{\substack{\hat \rho:\lbrack0,T\rbrack\to M:\\
            \hat\rho_0=\rho_0,\\
            T^{-1}\int_0^T\hat\rho_t\,\mathrm{d}t=\bar\rho}}
    \I_{\lbrack0,T\rbrack}(\hat\rho).
\label{eq:ergodic LDP fixed T}
\end{equation}
\label{cor:ergodic LDP fixed T}
\end{corollary}

\section{Connecting gradient flows to large deviations}
\label{sec:connection}

As stated in the Introduction, any evolution equation of gradient flow type in fact admits many other non-equivalent gradient flow structures \cite{dietert2015characterisation}. In the case of the Stein PDE~\eqref{eq:Stein pde} this phenomenon is exemplified by the structures proposed in \cite{liu2017stein,duncan2019geometry} and \cite{chewi2020svgd}. However, each gradient flow structure is related to a particular form of the noise in the corresponding interacting particle system. In this section we leverage our results from Sections \ref{sec:GF} and \ref{sec:LD} to make our Informal Result~\ref{res:ldp and GF} precise: the gradient flow structure from Section~\ref{sec:GF} corresponds to the noise described by the SDE~\eqref{eq:SDE}.
Our rigorous statement draws a connection between the reformulation of the gradient flow dynamics in terms of the energy-dissipation (in-)equality \eqref{eq:EDI} and the large-deviation functional \eqref{eq:LD I}:

\begin{theorem}[Connection between energy-dissipation and large deviations] For any curve $\rho:[0,T] \rightarrow M$ such that $t \mapsto \mathrm{KL}(\rho_t)$ is differentiable and \eqref{eq:conditions EDI} holds for all $t \in [0,T]$, the left-hand side of \eqref{eq:EDI} coincides with $\I_{[0,T]}(\rho)$ up to a factor of $\tfrac{1}{2}$, that is
\begin{equation}
    \I_{\lbrack0,T\rbrack}(\rho)=\mfrac12\mathrm{KL}(\rho_T) - \mfrac12\mathrm{KL}(\rho_0) + \frac{1}{4} \int_0^T\!\left\Vert \partial_t \rho_t \right\Vert^2_{T_{\rho_t} M}\,\dd t + \int_0^T\!\left\Vert \mfrac{\delta \tfrac12\mathrm{KL}}{\delta \rho} (\rho_t) \right\Vert_{T_{\rho_t}^* M}^2\,\dd t.
\label{eq:mpr decomposition}
\end{equation}
This implies that the large deviations from Theorem~\ref{th:large deviations} uniquely induce the gradient flow system with driving energy $\tfrac12\KL(\rho)$ and cotangent norm $\lVert\cdot\rVert_{T^*_\rho M}$ in the sense of~\cite{mielke2011gradient}.
\label{th:mpr}
\end{theorem}

Connections between energy-dissipation and large deviations have a long history in physics, starting from the idea that for non-evolving random systems, the Boltzmann-Gibbs-Helmholtz free energy $\F(\rho)$ of a macroscopic state $\rho$ is related to the probability of corresponding microstates through $\PP\super{N}\big(B_\epsilon(\rho)\big)\sim\exp\big(-N \mathcal{F}(\rho)/(\kappa_B T)\big)$. For the sake of brevity we ignore the Boltzmann constant $\kappa_B$ and the constant temperature $T$.
A dynamical version of this principle was proposed by Onsager and Machlup~\cite{Onsager1931I, Onsager1953MachlupI}, showing that for a number of physical examples with reversible randomness on the microscopic level, the path measures behave like
\begin{equation}
    \PP\super{N}\big(\Bcal_\epsilon(\rho)\big)\sim \exp\Big(-N\Big\lbrack
        \F(\rho_T) - \F(\rho_0)  + \tfrac12\int_0^T\lvert\partial_t\rho_t\rvert^2_{\rho_t}\,dt +\tfrac12 \int_0^T\!\lvert\tfrac{\delta\F}{\delta\rho}(\rho_t)\rvert^2_{\rho_t^*}\,\dd t
    \Big\rbrack\Big),
\label{eq:Onsager-Machlup}
\end{equation}
at least close to equilibrium. In the above display, 
$\F$ stands for an appropriate free energy functional, $\lvert\cdot\rvert_{\rho_t}$ and $\lvert\cdot\rvert_{\rho_t^*}$ for suitable dual norms, and $\varepsilon>0$ is assumed to be small. Moreover, Onsager and Machlup demonstrated that these constituents define a corresponding gradient flow structure ~\footnote{The reversibility of a Markov process is often called detailed balance in the physics literature to distinguish it from thermodynamical reversibility and was referred to as reciprocity relations by Onsager and Machlup \cite{Onsager1931I}. Moreover, they called the energy-dissipation inequality~\eqref{eq:EDI} the principle of least dissipation.}. Note that the exponent in~\eqref{eq:Onsager-Machlup} has the dimensions of a free energy (ignoring the Boltzmann constant and the constant temperature), which is consistent with the Boltzmann-Gibbs-Helmholtz free energy as described above.

More recently, this principle was extended to include more general dynamics that are also allowed to evolve far away from their equilibrium state~\cite{mielke2016generalization}. It turns out that for any microscopic \emph{reversible} Markov process, the corresponding large-deviations rate can be decomposed in such a way that it uniquely defines the free energy functional~$\F$ and the dissipation mechanism (in~\eqref{eq:Onsager-Machlup} encoded in the two norms $\lvert\cdot\rvert_{\rho_t}$ and $\lvert\cdot\rvert_{\rho_t^*}$) of a gradient flow.
For quadratic rate functionals, as in our case~\eqref{eq:LD I}, this decomposition corresponds to an expansion of squares, which basically amounts to connecting the energy-dissipation (in-)equality \eqref{prop:EDE} to the large-deviation functional~\eqref{eq:LD I}. This connection is the rigorous statement of the Onsager-Machlup principle described above, as well as of our Informal Result~\ref{res:ldp and GF}. 

\begin{proof}[Proof of Theorem \ref{th:mpr}]
The decomposition follows the same argument as the proof of Proposition~\ref{prop:EDE}. Note that by~\eqref{eq:Onsager isometry} and \eqref{eq:Onsager duality}, the two squared norms are convex duals to each other, i.e. for all $\rho\in M$, $\xi\in T_\rho M$ and $\phi\in T^*_\rho M$
\begin{align*}
    \mfrac14\left\Vert\xi \right\Vert_{T_{\rho}M}^2 = \sup_{\phi\in T_\rho^*M} \tensor[_{T_\rho^*M}]{\langle \phi,\xi\rangle}{_{T_\rho M}} - \left\Vert\phi \right\Vert_{T_{\rho}^* M}^2,
    &&\text{and}&&
    \left\Vert\phi \right\Vert_{T_{\rho}^* M}^2= \sup_{\xi\in T_\rho M} \tensor[_{T_\rho^*M}]{\langle \phi,\xi\rangle}{_{T_\rho M}} - \mfrac14\left\Vert\xi \right\Vert_{T_{\rho}M}^2.
\end{align*}
This indeed implies that~\eqref{eq:mpr decomposition} is a decomposition in the sense of \cite[eq.~(1.10)]{mielke2011gradient}. The uniquenes of the driving energy $\tfrac12\KL$ and cotangent norm $\lVert\cdot\rVert_{T_\rho^* M}$ follows from \cite[Th.~2.1(ii)]{mielke2011gradient}.
\end{proof}

\begin{remark}
Strictly speaking, this result yields a different gradient flow structure:
\begin{equation*}
    \partial_t\rho_t=-(2\KK_{\rho_t})\big(\tfrac12\KL(\rho_t)\big).
\end{equation*}
The constant $1/2$ in front of the Kullback-Leibner divergence is a known issue; it arises because the Kullback-Leibner divergence is related to the difference of large-deviation costs of moving forward and backward in time (note the time-reversal symmetry \eqref{eq:time-reversal}), hence when only moving forward in time, the constant $1/2$ appears. Similarly, the constant $2$ in front of the Onsager operator $\KK_\rho$ appears as the derivative of the norm $\lVert\cdot\rVert^2_{T_\rho^* M}$. We again refer to \cite{mielke2011gradient} for the details. Of course, one can also absorb the constant $1/2$ in the Onsager operator as we do. 
\label{rem:half}
\end{remark}

From Theorem~\ref{th:mpr} we immediately obtain the following relation between the Stein-Fisher information and free energy dissipation:
\begin{corollary}[{\cite{duncan2019geometry,lu2019scaling}}] For the solutions $\rho$ of the Stein PDE~\eqref{eq:Stein pde}:
\begin{equation*}
  \frac{\mathrm{d}}{\mathrm{d}t} \KL(\rho_t)=-I^k_\Stein(\rho_t).
\end{equation*}
\end{corollary}
Hence the Fisher information controls the convergence for the Stein PDE as $t\to\infty$ (see also the discussion of the Stein log-Sobolev inequality in \cite[Remark 35]{duncan2019geometry} and \cite{korba2020non}). In the next section we show that the Fisher information also controls the convergence for the stochastic SVGD scheme as both $N\to\infty$ and $T\to\infty$.

\section{Long-time behaviour and the Stein-Fisher information}
\label{sec:Fisher}

Generally speaking, large values of rate functionals promise fast convergence, as the corresponding fluctuations are suppressed.
To obtain interpretable information from \eqref{eq:LD I}, we study the rate functional $\bar\I_T$ governing the ergodic average (see Corollary \ref{cor:ergodic LDP fixed T}) in the regime where the final time $T$ is large. As mentioned in the introduction, the leading order term will be given by the Stein-Fisher information~\eqref{eq:Fisher} (or the kernelised Stein discrepancy). We first show this relation between the Stein-Fisher information and the large-deviation functional, and then investigate the Stein-Fisher information for different kernels. 

\subsection{From large deviations to the Stein-Fisher information}

Recall the large-deviation principle for the ergodic average from Corollary~\ref{cor:ergodic LDP fixed T}, for a fixed final time $T$. By the energy-dissipation decomposition~\eqref{eq:mpr decomposition} we may write, using a change of variables,
\begin{equation*}
    \bar\I_T(\bar{\rho})
        =
    \inf_{\substack{\hat \rho:\lbrack0,1\rbrack\to M:\\
            \hat\rho_0=\rho_0,\\
            \int_0^1\hat\rho_t\,\mathrm{d}t=\bar\rho}}
        \Bigg\{
        \frac12\mathrm{KL}(\hat\rho_1) - \frac12\mathrm{KL}(\rho_0) + \frac{1}{4T} \int_0^1\!\left\Vert \partial_t \hat\rho_t \right\Vert^2_{T_{\hat\rho_t} M}\,\dd t + \frac{T}{4}\int_0^1\!\left\Vert \mfrac{\delta \mathrm{KL}}{\delta \rho} (\hat\rho_t) \right\Vert_{T_{\hat\rho_t}^* M}^2\,\dd t
        \Bigg\}.
\end{equation*}
Therefore, at least formally, we see that the last term, representing the Stein-Fisher information (see Remark \ref{rem:Fisher cotangent}), becomes  dominant and of order $\mathcal{O}(T)$. To make this into a rigorous statement, one might naively take the pointwise limit of $T^{-1}\bar\I_T(\bar{\rho})$; however generally this limit does not exist, nor is it the right limit concept to use. To be consistent with the notion of large deviations we will need to use the concept of $\Gamma$-convergence~\cite{Braides2002}.
Together, the large-deviation principle and the $\Gamma$-convergence will then imply a joint large-deviation principle in $N$ and $T$, see, for example \cite[Sec.~4]{Bertinietal2007}. This will be the content of Corollary~\ref{cor:NT Stein-Fisher LDP}.

Let us stress here that the notion of $\Gamma$-convergence requires a topology on the underlying space, and that the most natural topology is the one for which the limit $4^{-1}I_k^\Stein$ has compact (sub-)level sets, see for example~\cite[Sec.~1.2]{DemboZeitouni1998} and \cite[Lem.~6.2]{Braides2002}. In our case, this means that we will choose the topology defined by \eqref{eq:ws topology}.

\begin{theorem} Fix the initial condition $\rho_0$ such that $\KL(\rho_0)<\infty$. Then in the topology of \eqref{eq:ws topology},
\begin{equation*}
    \mathop{\Gamma\!-\!\lim}_{T\to\infty} \, \frac1T\bar{\I}_T = \frac14 I^k_\Stein,
\end{equation*}
meaning that
\begin{enumerate}
    \item for all converging sequences of probability measures $\bar{\rho}_T\wsconv{T\to\infty}\bar{\rho}$,
    \begin{equation}
        \liminf_{T\to\infty} \frac1T\bar{\I}_T(\bar{\rho}_T)\geq\frac14 I^k_\Stein(\bar{\rho}), \qquad\text{and}
    \label{eq:Gamma lower bound}
    \end{equation}
    \item for all $\bar\rho\in M$, there exists a converging sequence of probability measures $\bar{\rho}_T\wsconv{T\to\infty}\bar{\rho}$ such that
    \begin{equation}
        \limsup_{T\to\infty} \frac1T\bar{\I}_T(\bar{\rho}_T)\leq\frac14 I^k_\Stein(\bar{\rho}).
    \label{eq:Gamma upper bound}
    \end{equation}
\end{enumerate} 
\label{th:large time Gamma}
\end{theorem}

\begin{proof}
    For the upper bound we take an arbitrary $\bar{\rho}$, for now assuming that $\KL(\bar{\rho})<\infty$. The statement~\eqref{eq:Gamma upper bound} would be trivial if we could replace the infimum in \eqref{eq:ergodic LDP fixed T} by the constant path $\bar{\rho}_t\equiv\bar{\rho}$. However this is likely to violate the initial condition, and so we first need to construct a finite-time and finite-cost connecting path between $\rho_0$ and $\bar{\rho}$. For this construction we shall need two ingredients. The first ingredient is the fact that $\KL$ is the `quasipotential', i.e. for all $\hat\rho\in M$,
    \begin{equation}
        \lim_{T\to\infty} \inf_{\substack{\rho:[0,T]\to M:\\ \rho_0=\pi,\rho_T=\hat\rho}} \I_{\lbrack0,T\rbrack}(\rho)=\KL(\hat\rho).
    \label{eq:quasipotential}
    \end{equation}
    This statement is standard and can be proven by solving the corresponding Hamilton-Jacobi-Bellman equation, see for example \cite{FreidlinWentzell2012}. The second ingredient is the so-called `time-reversal symmetry', meaning that for arbitrary $T>0$, path $\rho:\lbrack0,T\rbrack\to M$ and reversed path $\overleftarrow{\rho\!}_t:=\rho_{T-t}$,
    \begin{equation}
        \I_{\lbrack0,T\rbrack}(\rho)-\I_{\lbrack0,T\rbrack}(\overleftarrow{\rho})=\KL(\rho_T)-\KL(\rho_0).
    \label{eq:time-reversal}
    \end{equation}
    This symmetry is implied by the reversibility of the process $\rho\super{N}_t$ ~\cite[Th.~ 3.3]{mielke2011gradient}, but it can also be seen directly from the decomposition~\eqref{eq:mpr decomposition}.
    
    We now use these two ingredients to construct a connecting path between $\rho_0$ and $\bar{\rho}$. By \eqref{eq:quasipotential}, there exists a $T_1<\infty$ and a path $\rho^1:\lbrack 0,T_1\rbrack\to M$ connecting $\pi$ to $\rho_0$ so that $\I_{\lbrack0,T_1\rbrack}(\rho^1)\leq \KL(\rho_0)+1$. By the time-reversal symmetry~\eqref{eq:time-reversal}, the reversal $\overleftarrow{\rho^1}$ of this path connects $\rho_0$ to $\pi$, and satisfies $\I_{\lbrack0,T_1\rbrack}(\overleftarrow{\rho^1})\leq1$. Similarly, there exists a $T_2<\infty$ and a path $\rho^2:\lbrack0,T_2\rbrack\to M$ connecting $\pi$ to $\bar{\rho}$ such that $\I_{\lbrack0,T_2\rbrack}(\rho^2)\leq \KL(\bar\rho)+1$.
    
    From these two paths we construct a new, continuous path $\rho:\lbrack0,T\rbrack\to M$ for arbitrary large $T>0$:
    \begin{equation*}
        \rho_t:=
        \begin{cases}
            \overleftarrow{\rho^1_t}, &t\in\lbrack0,T_1),\\
            \rho^2_{t-T_1},             &t\in\lbrack T_1,T_1+T_2),\\
            \bar{\rho},                 &t\in\lbrack T_1+T_2,T\rbrack.
        \end{cases}
    \end{equation*}
    This path has the average value
    \begin{equation*}
        \bar{\rho}_T:=\frac1T\int_0^T\!\rho_t\,dt = \frac{T_1}{T}\Big({\textstyle\frac1T\int_0^T\!\rho^1_t\,dt}\Big) + \frac{T_2}{T}\Big({\textstyle\frac1T\int_0^T\!\rho^2_t\,dt}\Big) + \frac{T-T_1-T_2}{T}\bar{\rho},
    \end{equation*}
    which clearly converges as claimed, $\bar{\rho}_T\wsconv{}\bar{\rho}$.
    
    Plugging this path and average value into definitions~\eqref{eq:LD I}, \eqref{eq:ergodic LDP fixed T} and using \eqref{eq:Onsager isometry} yields 
    \begin{align*}
        \frac1T\bar{\I}_T(\bar{\rho}_T)
        &\leq
        \frac1T\I_{\lbrack0,T\rbrack}(\rho)
        =
        \frac1T\I_{\lbrack0,T_1\rbrack}(\overleftarrow{\rho^1}) 
        + \frac1T\I_{\lbrack T_1,T_1+T_2\rbrack}(\rho^2)
        + \frac1T\I_{\lbrack T_1+T_2,T\rbrack}(\bar{\rho})\\
        &\leq
        \frac{1}{T}
        + \frac{\KL(\bar{\rho})+1}{T} 
        + \frac{T-T_1-T_2}{T} \Big\lVert \frac{\delta\tfrac12\KL}{\delta\rho}(\bar{\rho})\Big\rVert^2_{T^*_{\bar{\rho}}M},
    \end{align*}
    and the upper bound~\eqref{eq:Gamma upper bound} follows by letting $T\to\infty$ together with the assumption $\KL(\bar{\rho})<\infty$.
    
    We now handle the case when $\KL(\bar{\rho})=\infty$ using an additional approximation and a diagonal argument. Without loss of generality we may assume that $I^k_\Stein(\bar{\rho})<\infty$, else the statement~\eqref{eq:Gamma upper bound} would be trivial. It follows from Definition~\ref{def:cotangent spaces} and Remark~\ref{rem:Fisher cotangent} that there exists a sequence $\phi^\epsilon\in C_c^{\infty}(\mathbb{R}^d)$ so that
    \begin{equation*}
        I^k_\Stein(\bar{\rho})
            =
        \big\lVert\mfrac{\bar{\rho}}{\pi}\big\rVert^2_{T^*_\pi M}
            \leftarrow
        \big\lVert\phi^\epsilon\big\rVert^2_{T^*_\pi M}
            =
        I^k_\Stein(\pi \phi^\epsilon).
    \end{equation*}
    Without loss of generality we may assume that $\pi\phi^\epsilon$ is a probability measure. Of course this sequence has uniformly bounded Fisher information, so that it has a convergent subsequence by Lemma~\ref{cor:compact level sets}. Let us relabel this sequence so that $\pi\phi^\epsilon\wsconv{}\bar\rho$. Clearly $\KL(\pi \phi^\epsilon)<\infty$ and so by the construction above there exists an approximating sequence $\bar{\rho}_T^\epsilon\rightharpoonup \pi\phi^\epsilon$ for which $\limsup_{T\to\infty} \frac1T\bar\I_T(\bar{\rho}_T^\epsilon)\leq \frac14 I^k_\Stein(\pi\phi^\epsilon)$. We can then define $\bar\rho_T:=\bar\rho_T^{\epsilon_T}$ where we pick $\epsilon_T\to0$ sufficiently slowly so that $\bar\rho_T\wsconv{}\bar{\rho}$ and
    \begin{equation*}
        \limsup_{T\to\infty} \frac1T\bar\I_T(\bar{\rho}_T)
            =
        \limsup_{\epsilon\to0} \limsup_{T\to\infty} \frac1T\bar\I_T(\bar{\rho}_T^\epsilon)    
            \leq
        \limsup_{\epsilon\to0} \frac14 I^k_\Stein(\pi\phi^\epsilon)
            =
        \frac14 I^k_\Stein(\bar\rho).
    \end{equation*}
    
    For the lower bound~\eqref{eq:Gamma lower bound}, pick an arbitrary convergent sequence $\bar{\rho}_T\rightharpoonup\bar{\rho}$, and for each $T>0$ an arbitrary path $\rho:\lbrack0,T\rbrack\to M$ starting from $\rho_0$ and with average value $T^{-1}\int_0^T\!\rho_t\,\dd t=\bar{\rho}_T$. We again use the decomposition~\eqref{eq:mpr decomposition} as well as \eqref{eq:Fisher equalities}, and neglecting some non-negative terms to derive
    \begin{equation*}
        \frac1T\I_{\lbrack0,T\rbrack}(\rho)
        \geq - \frac1{2T}\KL(\rho_0) + \frac1{4T} \int_0^T\!I^k_\Stein(\rho_t)\,\dd t
        \geq - \frac1{2T}\KL(\rho_0) + \frac1{4} \!I^k_\Stein(\bar{\rho}_T),
    \end{equation*}
    using Jensen's inequality and the fact that the Stein-Fisher information $I^k_\Stein$ is convex. By taking the infimum over all such paths we find $\frac1T\bar{\I}_T(\bar\rho_T)\geq - \frac1{2T}\KL(\rho_0) + \frac1{4} \!I^k_\Stein(\bar{\rho}_T)$. Then the lower bound~\eqref{eq:Gamma lower bound} follows from the lower semicontinuity of $\KL$ as a consequence of Corollary~\ref{cor:compact level sets}.
\end{proof}

The following result is the mathematically precise statement of our Informal Result~\ref{res:informal}.
\begin{corollary} 
\label{cor:Fisher LDP}
The ergodic average empirical measure $\bar{\rho}_T\super{N}$ associated to the SDE~\eqref{eq:SDE} satisfies the large-deviation principle as first $N\to\infty$ and then $T\to\infty$ with rate functional $\tfrac14 I^k_\Stein$, i.e.
\begin{equation*} 
    \PP\super{N}\big(\bar{\rho}\super{N}_T \in \Bcal_\epsilon(\bar{\rho})\big)\sim \exp\big(-\mfrac14 NT I^k_\Stein(\bar\rho)\big).
\end{equation*}
\label{cor:NT Stein-Fisher LDP}
\end{corollary}

\subsection{Comparing the Stein-Fisher information for different kernels}

\label{sec:comparing fisher}

Corollary \ref{cor:Fisher LDP} motivates using the Stein-Fisher information $I^k_{\mathrm{Stein}}$ for a principled choice of the kernel $k$ (greater values of $I^k_{\mathrm{Stein}}$ promise faster convergence). As stated in Proposition \ref{prop:Fisher comparison}, the comparison between $I^{k_1}_{\mathrm{Stein}}$ and $I^{k_2}_{\mathrm{Stein}}$ can be made on the basis of the RKHSs $\mathcal{H}_{k_1}$ and $\mathcal{H}_{k_2}$. Here we provide the proof based on the duality relations established in Section \ref{subsec:cotangent}.

\begin{proof}[Proof of Proposition \ref{prop:Fisher comparison}]
In this proof, we use the notation $T_\rho M_1$ and $T_\rho M_2$ to distinguish the tangent spaces induced by $k_1$ and $k_2$, respectively, and employ a similar convention for the cotangent spaces. 
We first show that \ref{it:RKHS balls}.) implies \ref{it:Fish comparison}.):
By Remark \ref{rem:Fisher cotangent}, it is sufficient to show that $T^*_\rho M_2 \subset T^*_\rho M_1$, with 
$$
\left\Vert \phi \right\Vert^2_{T_\rho^*M_1} \le \left\Vert \phi \right\Vert^2_{T_\rho^*M_2},
$$
for all $\phi \in T_\rho^*M_2$. Now, for $\phi \in C_c^\infty(\mathbb{R}^d)$, $\rho \in M$, and $i \in \{1,2\}$, we see that
\begin{equation}
    \Vert \phi \Vert_{T_\rho^*M_i} = \sup_{0 \neq \psi \in T_{\rho} M_i} \frac{\tensor[_{T_\rho^*M_i}]{\langle \phi,\psi\rangle}{_{T_\rho M_i}}}{\Vert \psi \Vert_{T_\rho M_i} }
    = \sup_{v \in \overline{\T_{k_i,\rho} \nabla C_c^{\infty}(\mathbb{R}^d)}^{\mathcal{H}^d_{k_i}}} \frac{\int_{\mathbb{R}^d} v \cdot \nabla \phi \, \mathrm{d}\rho}{\Vert v \Vert_{\mathcal{H}^d_{k_i}}} = \sup_{v \in \mathcal{H}^d_{k_i}} \frac{\int_{\mathbb{R}^d} v \cdot \nabla \phi \, \mathrm{d}\rho}{\Vert v \Vert_{\mathcal{H}^d_{k_i}}}, 
\end{equation}
where the first equality follows from the duality between $T_\rho M_i$ and $T^*_\rho M_i$, the second equality follows directly from Definition \ref{def:Riemann geometry}, and the third equality is a consequence of the Helmholtz decomposition in Proposition \ref{prop:helmholtz}. The claim now follows from the fact that by construction, $C_c^{\infty}(\mathbb{R}^d)$ is dense in $T_\rho^*M_1$ and $T_\rho^*M_2$.

To show that \ref{it:Fish comparison}.) implies \ref{it:RKHS balls}.), assume that $v = \T_{k,\rho}\nabla \phi$ and $\xi + \nabla \cdot (\rho v) = 0$. We then have
\begin{equation}
    \Vert v \Vert_{\mathcal{H}_k^d} = \Vert \xi \Vert_{T_\rho M} = \sup_{\psi \in T^*_\rho M}
    \frac{\tensor[_{T_\rho^*M_i}]{\langle \xi,\psi\rangle}{_{T_\rho M_i}}}{\Vert \psi \Vert_{T^*_\rho M_i} }
    = \sup_{\psi \in T_\rho^* M} \frac{\int_{\mathbb{R}^d} v \cdot \nabla \psi \, \mathrm{d}\rho}{\Vert \psi \Vert_{T_\rho^*M}},
\end{equation}
and the statement follows by similar arguments as above.
\end{proof}

To conclude this section, we cite Lemma 42 from \cite{duncan2019geometry}, illustrating some consequences of Corollary \ref{cor:NT Stein-Fisher LDP} and Proposition \ref{prop:Fisher comparison}: 
\begin{example}
	\label{lem:exp p kernel}
	Consider the positive definite kernels $k_{p,\sigma}:\mathbb{R}^d \times \mathbb{R}^d \rightarrow \mathbb{R}$, defined via
\begin{equation}
	\label{eq:p kernel}
	k_{p,\sigma}(x,y) = \exp\left(-\frac{\vert x - y \vert^p}{\sigma^p}\right),
\end{equation}
where $p \in (0,2]$ is a smoothness parameter, and $\sigma > 0$ controls the kernel width. Then, following \cite[Lemma 42]{duncan2019geometry}, $k_{p,\sigma}$ is integrally strictly positive definite (see Assumption \ref{ass:ISPD}). Furthermore, the associated RKHSs  are nested according to the regularity of the corresponding kernels: If $p > q$, then $\mathcal{H}_{k_{p,\sigma_p} }\subset \mathcal{H}_{k_{q,\sigma_q}}$ (with strict inclusion), for all $\sigma_p,\sigma_q > 0$. The inclusion of unit balls, that is, 
\begin{equation}
\Vert \phi \Vert_{\mathcal{H}_{k_{q,\sigma_q}}} \le \Vert \phi \Vert_{\mathcal{H}_{k_{p,\sigma_p}}}, \qquad \qquad \phi \in \mathcal{H}_{k_{p,\sigma_p}},
\end{equation}
relevant for Proposition \ref{prop:Fisher comparison} can moreover be obtained by a suitable choice of the kernel widths $\sigma_q$ and $\sigma_p$.
Consequently, combining Proposition \ref{prop:Fisher comparison} and Corollary \ref{cor:Fisher LDP}, kernels with lower regularity are expected to incur faster convergence of the ergodic limit \eqref{eq:ergodic limit} for the SDE system \eqref{eq:SDE}, asymptotically in the regime when $N$ and $T$ are large. The performance of numerical algorithms based on different choices of $k$ is not straightforward, as the stiffness of the SDE (and corresponding time discretisations) have to be taken into account. To illustrate our findings, we instead consider fixed points of the ODE system \eqref{eq:ODE} obtained for $ t \rightarrow \infty$, see Figure \ref{fig:fix point}. The approximation of the target $\pi$ obtained using the low-regularity Laplace kernel ($p=1$) appears to be more regular and more evenly spaced in comparison with the approximation obtained using the high-regularity squared exponential kernel ($p=2$). 

On a heuristic level, we can connect these observations to our results as follows: The large-deviation functionals $\I_{\lbrack0,T\rbrack}$ and  $\bar{\I}_T$ quantify the speed of convergence as solutions of the SDE system \eqref{eq:SDE} approach solutions of the Stein PDE \eqref{eq:Stein pde} as $N \rightarrow \infty$.
Recall from Section \ref{subsec:statphys} that the SDE \eqref{eq:SDE} preserves the extended target $\bar{\pi}$ for \emph{any $N \in \mathbb{N}$}, and that solutions of the ODE system \eqref{eq:ODE} solve the Stein PDE \eqref{eq:Stein pde} in a weak sense. Therefore, our results suggests that the ODE \eqref{eq:ODE} provides approximations of the SDE \eqref{eq:SDE} (and hence, the target $\pi$) that are expected to be more satisfactory if  $\I_{\lbrack0,T\rbrack}$ and  $\bar{\I}_T$ are large. We stress that this line of argument is heuristic and should be treated as a conjecture, since our rigorous results concern the SDE \eqref{eq:SDE} and not the ODE \eqref{eq:ODE}. Understanding the finite-particle regime of the ODE \eqref{eq:ODE}, and possible connections to large-deviation principle remains an interesting subject for future research. 
\end{example}

\section{Conclusion and outlook}
\label{sec:examples}

In this paper, we have drawn connections between the \emph{variational inference}-type ODE \eqref{eq:ODE} 
and the
\emph{Markov Chain Monte Carlo}-type SDE \eqref{eq:SDE} based on  gradient flow structures and large-deviation functionals. Extending previous works, our results take a step towards a quantitative understanding of the mean-field limit of SVGD. In particular, in the regime when $N$ and $t$ are large, the convergence towards the target $\pi$ is governed by the Stein-Fisher information (or kernelised Stein discrepancy). The relationship between variational inference, Markov Chain Monte Carlo and ideas from statistical physics promises to be a fruitful direction for future research beyond SVGD. As our results are asymptotic, quantifying the accuracy of SVGD for the practically relevant scenario of small $N$ and $t$ remains a challenging and open problem.

\begin{figure}
	\centering
	\begin{subfigure}{.5\textwidth}
		\centering
		\includegraphics[width=1.0\linewidth]{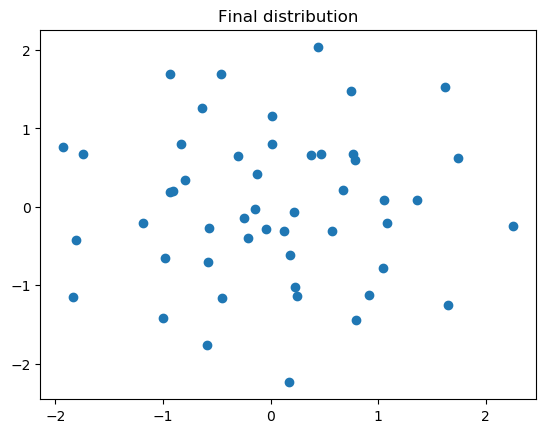}
		\caption{squared exponential kernel, $p=2$, $\sigma = 1$.}
		\label{fig:sub1}
	\end{subfigure}%
	\begin{subfigure}{.5\textwidth}
		\centering
		\includegraphics[width=1.0\linewidth]{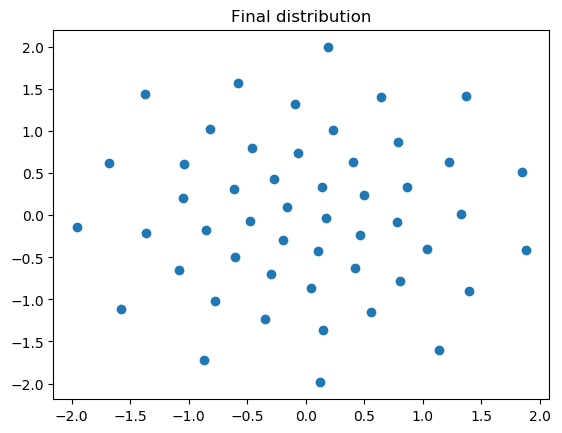}
		\caption{Laplace kernel, $p=1$, $\sigma = 1$.}
		\label{fig:sub2}
	\end{subfigure}
	\caption{Approximations of a two-dimensional standard normal distribution using deterministic SVGD based on the  ODE \eqref{eq:ODE} and two different positive definite kernels $k_{p,\sigma}$.}
	\label{fig:fix point}
\end{figure}

\section*{Acknowledgements}
This research has been funded by 
Deutsche Forschungsgemeinschaft (DFG) through the grant 
CRC 1114 \lq Scaling Cascades in Complex Systems\rq \,(projects A02 and C08, project number 235221301).

\appendix

\section{Proofs for Section \ref{sec:LD}}
\label{app:proofs}

\begin{proof}[Proof of Lemma~\ref{lem:generator}]
We shall only prove the claim for a large class of test functions of the form:
\begin{equation}
  F(\rho) = \phi \big( \langle p_1,\rho \rangle , \ldots, \langle p_L, \rho \rangle \big),
\label{eq:test functions}
\end{equation}
for arbitrary $L\in\Nb$, $p_1,\hdots,p_L\in C_b^2(\R^d)$ and $\phi\in C_b^2(\R^L)$, where $\langle p_i, \rho \rangle = \int_{\mathbb{R}^d} p_i \, \mathrm{d}\rho$. Applied to the empirical measure~\eqref{eq:empirical measure}, these test functions become:
\begin{equation*}
    F(\rho_t^N) = \phi\left( \frac{1}{N} \sum_{i=1}^N p_1 (X_t^i), \ldots, \frac{1}{N}\sum_{i=1}^N p_L(X_t^i) \right) =: G(X_t^1, \ldots, X_t^N)=:G(\bar{X}_t).
\end{equation*}
A straightforward application of It{\^o}'s Lemma to the process~\eqref{eq:SDE} gives, abbreviating the martingale $\mathrm{d}M_t:=\nabla G(\bar{X}_t)^T\sqrt{2K(\bar{X}_t)}\,\mathrm{d}W_t$,
\begin{align*}
	\mathrm{d}G(\bar{X}_t) & = \sum_{i=1}^N \nabla_{X^i_t}G(\bar{X}_t) \cdot A_{i}(\bar{X}_t)\,\mathrm{d}t + \mfrac1N\sum_{i,j=1}^N k(X^i_t,X^j_t)\nabla^2_{ij}G(\bar{X}_t)\,\mathrm{d}t + \mathrm{d}M_t\\
	& = \mfrac{1}{N} \sum_{i=1}^N \sum_{l=1}^L \partial_l \phi \, \nabla p_l(X_t^i)\cdot A_{i}(\bar{X}_t) \,\mathrm{d}t
 + \mfrac{1}{N^3}\sum_{i,j=1}^N \sum_{l,m = 1}^L \partial^2_{lm} \phi  \, \nabla p_l(X_t^i)\cdot\nabla p_m(X_t^j) k(X_t^i,X_t^j) \,\mathrm{d}t \\
	& \qquad + \mfrac{1}{N^2} \sum_{l=1}^L \sum_{i=1}^N\partial_l \phi \, \Delta p_l(X_t^i) k(X_t^i,X_t^i) \, \mathrm{d}t  + \mathrm{d} M_t,
\end{align*}
denoting 
\begin{equation*}
    A_i(\bar{x}) = \mfrac1N \sum_{j=1}^N \left( -k(x_i,x_j) \nabla V(x_j) + \nabla_{x_j} k(x_i,x_j)\right), \qquad \bar{x} = (x_1,\ldots,x_N). 
\end{equation*}
Notice that
\begin{equation*}
    \frac{\delta F}{\delta \rho}(\rho)(x) = \sum_{l=1}^L \partial_l \phi \, p_l(x), \qquad \frac{\delta^2 F}{\delta \rho^2}(\rho)(x,y) = \sum_{l,m=1}^L \partial^2_{lm} \phi \, p_l(x) p_m(y).
\end{equation*}
By taking the expectation, the martingale term drops out, so that
\begin{align*}
	\frac{\mathrm{d}}{\mathrm{d}t}\mathbb{E} F(\rho_t^N) & = \mathbb{E} \left[ \iint\! \big\lbrack - k(x,y) \nabla V(y) + \nabla_{y} k(x,y) \big\rbrack\cdot \nabla_x\left( \frac{\delta F}{\delta \rho}(\rho_t^N)\right) (x) \, \rho_t^N (\mathrm{d}x)\,\rho_t^N(\mathrm{d}y)\right]
	\\
	&\quad + \mfrac{1}{N}\mathbb{E} \Big[ \iint\! k(x,y) \nabla_x\cdot \nabla_y\left(  \frac{\delta^2 F }{\delta \rho^2}(\rho_t^N)\right)(x,y) \rho_t^N(\mathrm{d}x) \rho_t^N(\mathrm{d}y) \\
	&\hspace{6cm} + \int\! k(x,x) \Delta \! \left( \frac{\delta F}{\delta \rho}(\rho_t^N)\right)\!(x)\,\rho_t^N(\mathrm{d}x) \Big],
\end{align*}
which proves the claim (for test functions of the form~\eqref{eq:test functions}).
\end{proof}

\begin{proof}[Proof of Lemma~\ref{lem:limit Hamiltonian}]
First notice that
\begin{align*}
    \frac{\delta}{\delta \rho} \left(e^{N G} \right)(\rho)(x) 
    &=
    N e^{NG(\rho)} \frac{\delta G}{\delta \rho} (\rho) (x),
    \qquad\text{and}\\
    \frac{\delta^2}{\delta \rho^2} \left( e^{N G}\right)(\rho) (x,y)
    &=
    N^2 e^{N G(\rho)} \frac{\delta G}{\delta \rho} (\rho) (x) \frac{\delta G}{\delta \rho} (\rho) (y) + N e^{NG(\rho)} \frac{\delta^2 G}{\delta \rho^2}(\rho) (x,y).
\end{align*}
Therefore
\begin{subequations}
\begin{align*}
	(\Ham\super{N} G)(\rho)  &\stackrel{\eqref{eq:nonlinear semigroup}}{=} \iint_{\R^d\times\R^d}\!\big\lbrack - k(x,y) \nabla V(y) + \nabla_{y} k(x,y) \big\rbrack \cdot\left(\nabla_x\frac{\delta G}{\delta \rho}(\rho)(x)\right) \rho(\mathrm{d}x)\, \rho(\mathrm{d}y) \\
	& + \iint_{\R^d\times\R^d} \! k(x,y) \left(\nabla_x \frac{\delta G}{\delta \rho}(\rho) (x)\right) \cdot \left(\nabla_y \frac{\delta G}{\delta \rho}(\rho) (y)\right) \rho(\dd x)\, \rho(\dd y) + \mathcal{O}(N^{-1}).
\end{align*}
\end{subequations}
Assuming that $\rho$ is regular enough, we can write
\begin{align*}
    \int_{\R^d}\!\big\lbrack - k(x,y) \nabla V(y) + \nabla_{y} k(x,y) \big\rbrack \, \rho(\mathrm{d}y)
    =
    - \int_{\R^d}\! k(x,y) \big\lbrack \nabla V(y) + \nabla \log \rho (y)\rbrack \, \rho(\mathrm{d}y)
    =
    -\int_{\R^d}\! k(x,y)\frac{\delta\KL}{\delta\rho}(\rho)(y)\,\rho(\dd y).
\end{align*}
Then we see that
\begin{align*}
	\Ham  \left(\rho, \frac{\delta G}{\delta \rho} \right) & := \lim_{N \rightarrow \infty} (\Ham\super{N} G)(\rho)  \\
	& = - \iint_{\R^d\times\R^d}\!k(x,y) \Big\lbrack \big(\nabla \frac{\delta \KL}{ \delta \rho} (\rho)(x)\big)  \cdot \big(\nabla\!\frac{\delta G}{\delta \rho}(\rho)(y)\big) + \big(\nabla \frac{\delta G}{\delta \rho}(\rho) (x)\big) \cdot \big(\nabla \frac{\delta G}{\delta \rho}(\rho) (y)\big) \,\Big\rbrack\,\rho(\dd x)\, \rho(\dd y)
	\\
	& = - \left\langle \frac{\delta \KL}{\delta \rho}, \frac{\delta G}{\delta \rho} \right\rangle_{T_\rho^*M} + \left\Vert \frac{\delta F}{\delta \rho}\right\Vert_{T_\rho^*M}^2.
\end{align*}

\end{proof}

\bibliographystyle{abbrv}
\bibliography{refs_final}

\end{document}